\documentclass[11pt]{article}
\usepackage[letterpaper, margin=1in]{geometry}

\usepackage[utf8]{inputenc} % allow utf-8 input
\usepackage[T1]{fontenc}    % use 8-bit T1 fonts
\usepackage{url}            % simple URL typesetting
\usepackage{booktabs}       % professional-quality tables
\usepackage{amsfonts}       % blackboard math symbols
\usepackage{nicefrac}       % compact symbols for 1/2, etc.
\usepackage{microtype}      % microtypography
\usepackage{xcolor}         % colors
\usepackage{graphicx}
\usepackage{natbib}

\definecolor{DarkRed}{RGB}{139, 0, 0}
\definecolor{DarkBlue}{RGB}{0, 0, 139}

\usepackage{algorithm}
\usepackage{algorithmic}

\renewcommand{\algorithmicrequire}{\textbf{Input:}}
\renewcommand{\algorithmicensure}{\textbf{Output:}}

\usepackage{amsmath}
\usepackage{amssymb}
\usepackage{mathtools}
\usepackage{amsthm}

\usepackage{hyperref}       % hyperlinks (loaded after float packages)
\usepackage[capitalize,noabbrev]{cleveref}

\usepackage{dsfont}
\usepackage[mathscr]{euscript}
\usepackage{expl3}
\usepackage{colortbl}
\usepackage{multirow}
\usepackage{accents}

\usepackage{pifont}% http://ctan.org/pkg/pifont

\usepackage{mathpazo}

\usepackage{MnSymbol}
\DeclareMathAlphabet\mathbb{U}{msb}{m}{n}
\usepackage{xpatch}

\def\Rset{\mathbb{R}}

\let\P\undefined

\DeclareMathOperator*{\P}{\mathbb{P}}
\DeclareMathOperator*{\E}{\mathbb E}
\DeclareMathOperator*{\argmax}{argmax}

\DeclareMathOperator*{\diam}{diam}

\DeclareMathOperator{\sign}{sign}

\DeclareMathOperator{\Var}{Var}

\DeclarePairedDelimiter{\abs}{\lvert}{\rvert} 
\DeclarePairedDelimiter{\bracket}{[}{]}
\DeclarePairedDelimiter{\curl}{\{}{\}}
\DeclarePairedDelimiter{\paren}{(}{)}
\DeclarePairedDelimiter{\norm}{\|}{\|}

\ExplSyntaxOn
\tl_const:Nn \c_my_uc_alphabet_tl { ABCDEFGHIJKLMNOPQRSTUVWXYZ }
\tl_const:Nn \c_my_full_alphabet_tl { ABCDEFGHIJKLMNOPQRSTUVWXYZ abcdefghijklmnopqrstuvwxyz }

\tl_map_inline:Nn \c_my_uc_alphabet_tl
 { \cs_gset:cpn { c#1 } { \mathcal{#1} } }

\tl_map_inline:Nn \c_my_uc_alphabet_tl
 { \cs_gset:cpn { s#1 } { \mathscr{#1} } }

\tl_map_inline:Nn \c_my_full_alphabet_tl
 {
  \str_if_eq:nnTF { b#1 } { bf } 
    { } % Do nothing for \bf
    {
      \str_if_eq:nnTF { sf#1 } { sf }
        { } % Do nothing for \sf
        {
          \cs_gset:cpn { b#1 } { \ensuremath{\mathbf{#1}} } 
          \cs_gset:cpn { sf#1 } { \ensuremath{\mathsf{#1}} }
        }
    }
 }
\ExplSyntaxOff

\newcommand{\h}{\widehat}
\newcommand{\ov}{\overline}
\newcommand{\wt}{\widetilde}

\newcommand{\ignore}[1]{}

\newcommand{\1}{\mathds{1}}
\newcommand{\hh}{{\mathsf{h}}}
\newcommand{\tar}{\textrm{tar}}
\newcommand{\sur}{\textrm{sur}}
\newcommand{\loss}{\sfL}

\hypersetup{
  breaklinks   = true, %splits links across lines
  colorlinks   = true, %Colours links instead of ugly boxes
  urlcolor     = blue, %Colour for external hyperlinks
  linkcolor    = blue, %Colour of internal links
  citecolor    = blue %Colour of citations
}

\theoremstyle{plain}
\newtheorem{theorem}{Theorem}[section]
\newtheorem{proposition}[theorem]{Proposition}
\newtheorem{lemma}[theorem]{Lemma}
\newtheorem{corollary}[theorem]{Corollary}
\theoremstyle{definition}

\theoremstyle{remark}
\newtheorem{remark}[theorem]{Remark}

\title{Linear-Core Surrogates: Smooth Loss Functions with Linear\\
  Rates for Classification and Structured Prediction}
\author{
  Mehryar Mohri\\
  Google Research \& Courant Institute\\
  New York, NY\\
  \texttt{mohri@google.com}
  \and
  Yutao Zhong\\
  Google Research\\
  New York, NY\\
  \texttt{yutaozhong@google.com}
}
\date{}

\begin{document}

\maketitle

\begin{abstract}
  A fundamental dichotomy in the theory of classification sets smoothness
  against statistical efficiency: smooth surrogate losses such as the logistic
  loss enable fast $O(1/T)$ optimization but yield slow square-root
  $\sH$-consistency bounds, while piecewise-linear losses like the Hinge loss
  achieve optimal linear $\sH$-consistency rates but are non-differentiable.  We
  introduce \emph{Linear-Core (LC) Surrogates}, the first family of
  \emph{explicit} convex loss functions that provably resolve this tension.  By
  stitching a linear core to a smooth tail, we construct surrogates that are
  differentiable everywhere ($C^1$, and even $C^2$ under mild conditions) while
  retaining strict linear $\sH$-consistency bounds, the strongest known form of
  consistency guarantee.
  We establish these linear bounds across three increasingly complex settings:
  binary classification, multi-class classification, and structured prediction.
  To our knowledge, this is the first explicit construction to simultaneously
  achieve smoothness and linear $\sH$-consistency in any of these settings.
  Beyond their theoretical appeal, Linear-Core Surrogates offer practical
  advantages.  In multi-class classification, their constant gradient profile
  near the decision boundary provides natural robustness to instance-dependent
  label noise, outperforming Cross-Entropy by \emph{2.6\%} on corrupted
  CIFAR-10.  In structured prediction, their smoothness enables an unbiased
  stochastic gradient estimator that bypasses the $O(|\sY|^2)$ per-step
  complexity of exact inference, yielding a \emph{23$\times$ speedup} over
  Structured SVMs on large-vocabulary sequence tagging tasks.
\end{abstract}

\clearpage
\tableofcontents
\clearpage

%%%%%%%%%%%%%%%%%%%%%%%%%%%%%%%%%%%%%%%%%%%%%%%%%%%%%%%%%%%%%%%%%%%%%%%%%%%%%%%
\section{Introduction}
\label{sec:intro}
%%%%%%%%%%%%%%%%%%%%%%%%%%%%%%%%%%%%%%%%%%%%%%%%%%%%%%%%%%%%%%%%%%%%%%%%%%%%%%%

In the theory of classification, the choice of loss function is governed by two
conflicting desiderata: computational tractability and statistical consistency.
On one hand, practical optimization requires loss functions that are convex and
smooth, enabling the use of efficient gradient-based algorithms with fast
convergence rates ($O(1/T)$ or better) \citep{nesterov1983method,beck2009fast}.
On the other hand, theoretical guarantees often rely on \emph{consistency
  bounds}, which relate the excess error of the surrogate to that of the
discrete target loss, e.g., the 0-1 loss.  While classical
\emph{Bayes-consistency} ensures convergence in the infinite-sample limit and
for the family of all measurable functions
\citep{zhang2004statistical,lin2004note,bartlett2006convexity,
  steinwart2007compare}, recent work has focused on the stronger notion of
\emph{$\sH$-consistency}, which provides non-asymptotic guarantees restricted to
the hypothesis set $\sH$ of interest
\citep*{Awasthi2022Hconsistency,awasthi2022multi, MaoMohriZhong2023structured,
  MaoMohriZhong2024}.

Historically, a dichotomy has persisted in this landscape: smooth losses like
the Logistic (Cross-Entropy) or Exponential loss yield valid consistency bounds,
but these bounds are notoriously slow.  Specifically, due to the vanishing
curvature of these losses near the origin, the transfer rate is on the order of
the square root of the excess surrogate error
($\Delta \sR \propto \sqrt{\Delta \sL}$)
\citep{bartlett2006convexity,zhang2004statistical}.  This implies that high
precision in optimization translates inefficiently to target accuracy.
Conversely, piecewise-linear losses like the Hinge loss offer fast linear
consistency rates ($\Delta \sR \propto \Delta \sL$)
\citep{steinwart2007compare,Awasthi2022Hconsistency} but suffer from
non-differentiability, leading to optimization instability and slower
sub-gradient convergence rates ($O(1/\sqrt{T})$) (see
Table~\ref{tab:loss_comparison} for a comparison).

\begin{table}[h!]
  \centering
  \caption{Comparison of surrogate loss properties.}
  \label{tab:loss_comparison}
  \vspace{2pt}
  \begin{tabular}{lcccc}
    \toprule
    \textbf{Loss Function} & \textbf{Convexity} & \textbf{Smoothness}
    & \textbf{Consistency Bound} & \textbf{Rate} \\
    \midrule
    Hinge Loss & Yes & No (Non-diff.\ at $\pm 1$) & Linear & Fast \\
    Squared-Hinge & Yes & $C^1$ & Square-root & Slow \\
    Logistic / Exp & Yes & $C^\infty$ & Square-root & Slow \\
    \midrule
    \textbf{Linear-Core ($\ov{\Phi}$)} & \textbf{Yes} & \textbf{$C^1$ / $C^2$}
    & \textbf{Linear} & \textbf{Fast} \\
    \bottomrule
  \end{tabular}
\end{table}

In this work, we propose \emph{Linear-Core (LC) Surrogates}, a new family of
\emph{explicit} smooth loss functions designed to resolve this dichotomy.  By
constructing a loss with a linear ``core'' stitched to a smooth tail, we achieve
the best of both worlds: the fast $O(1/T)$ optimization rates of smooth losses
and the optimal linear $\sH$-consistency bounds of the Hinge loss.  To our
knowledge, this is the first explicit construction of smooth surrogates with
provable linear $\sH$-consistency bounds.

We extend the Linear-Core framework to the challenging domain of structured
prediction, where the output space $\sY$ is exponentially large.  Standard
surrogates such as the Structured SVM (SSVM) loss
\citep{tsochantaridis2005large} or the CRF negative log-likelihood
\citep{lafferty2001conditional} often lack consistency guarantees with respect
to discrete target metrics like the Hamming loss
\citep{ciliberto2016consistent,osokin2017structured,nowak2019sharp}.
\citet{MaoMohriZhong2023structured} addressed this theoretical gap by providing
a detailed analysis of $\sH$-consistency for structured prediction and
introducing a family of structured losses with provable $\sH$-consistency
guarantees, though with square-root bounds.  Building on this foundation, we
show that our Linear-Core surrogates not only preserve these rigorous guarantees
but \emph{strengthen} them to \emph{linear} $\sH$-consistency bounds.  As an
additional practical benefit, their smoothness enables an unbiased stochastic
gradient estimator that bypasses the $O(|\sY|^2)$ bottleneck of exact inference
in standard SSVM and CRF methods.

Our approach is most closely related to the work of \citet{cao2025establishing},
which also seeks linear convergence rates via \emph{Convolutional Fenchel-Young
  losses}.  While their resulting binary loss profile structurally resembles our
construction, their general framework defines losses \emph{implicitly} via
variational optimization, relies on non-standard decoding, and is limited to
excess loss bounds with respect to the family of all measurable functions
(Bayes-consistency).  In contrast, our Linear-Core surrogates are
\emph{explicit}, standard convex functions compatible with standard $\argmax$
decoding, and we provide the strictly stronger \emph{$\sH$-consistency} bounds.
Furthermore, our framework extends to multi-class and structured prediction with
closed-form losses, and the explicit formulation enables our unbiased stochastic
sampling algorithm (Section~\ref{sec:struct_optimization}), which bypasses the
$O(|\sY|^2)$ bottleneck of exact inference.  See Section~\ref{sec:related}
for an extended discussion.

\paragraph{Contributions and organization.}
The remainder of this paper follows a deliberate progression of the label space
complexity, where each setting builds upon the previous one.  Our primary
contribution is theoretical: we resolve a fundamental open question by
constructing the first explicit smooth surrogates with linear
$\sH$-consistency, and we do so across all three settings.  All proofs are
given in full in the main text, immediately following the corresponding
statements.
\begin{itemize}
\item \emph{Binary Classification (Section~\ref{sec:binary}):} We establish the
  fundamental mechanism in the simplest setting ($|\sY|=2$), constructing the
  Linear-Core surrogate by stitching a linear core to a smooth tail and proving
  that it achieves both $C^1$/$C^2$ smoothness and linear $\sH$-consistency
  bounds (Section~\ref{sec:binary-bound}).  To our knowledge, this is the first
  explicit smooth surrogate with provable linear $\sH$-consistency.

\item \emph{Multi-Class Classification (Section~\ref{sec:multi}):} We scale to
  $n$ competing classes via sum-losses, proving that the multi-class objective
  inherits the optimal linear transfer bound (Section~\ref{sec:multi-bound}).
  The constant gradient profile near the decision boundary acts as a natural
  `hard' regularizer, yielding $+2.6\%$ improvement over Cross-Entropy on
  CIFAR-10 under instance-dependent noise (Section~\ref{sec:multi_experiments}).

\item \emph{Structured Prediction (Section~\ref{sec:struct}):} We tackle the
  most complex domain where the output space is exponentially large, proving
  linear $\sH$-consistency bounds for structured sum-losses
  (Section~\ref{sec:struct-bound}).  The smoothness of our surrogates enables
  an unbiased stochastic gradient estimator with bounded variance independent
  of $|\sY|$ (Theorem~\ref{thm:variance_bound}), which bypasses the
  $O(|\sY|^2)$ bottleneck of exact inference, achieving a 23$\times$ speedup
  over SSVM and 17.4$\times$ over CRF (Sections~\ref{sec:struct_experiments}
  and~\ref{sec:efficiency_exp}).
\end{itemize}
Section~\ref{sec:related} discusses related work in detail, including the
relationship of our results with the universal square-root growth rate of
\citet{MaoMohriZhong2024}, and Section~\ref{sec:conclusion} concludes.

%%%%%%%%%%%%%%%%%%%%%%%%%%%%%%%%%%%%%%%%%%%%%%%%%%%%%%%%%%%%%%%%%%%%%%%%%%%%%%%
\section{Preliminaries}
\label{sec:pre}
%%%%%%%%%%%%%%%%%%%%%%%%%%%%%%%%%%%%%%%%%%%%%%%%%%%%%%%%%%%%%%%%%%%%%%%%%%%%%%%

We denote the input space by $\sX$, the label space by $\sY$, the hypothesis set
by $\sH$ and the data distribution by $\sD$. We consider a \emph{target loss}
$\ell_{\tar} \colon \sH \times \sX \times \sY \to \Rset$ (e.g., the 0-1 loss)
and a \emph{surrogate loss}
$\ell_{\sur} \colon \sH \times \sX \times \sY \to \Rset$ (e.g., a convex margin
loss).  The \emph{generalization error} of a hypothesis $h \in \sH$ is defined
as $\sE_\ell(h) = \E_{(x, y) \sim \sD}[\ell(h, x, y)]$.  The \emph{best-in-class
  error} is denoted by $\sE^*_\ell(\sH) = \inf_{h \in \sH} \sE_\ell(h)$.  The
difference $\sE_{\ell}(h) - \sE^*_{\ell}(\sH)$ is referred to as the
\emph{estimation error}.

\paragraph{$\sH$-consistency bounds.}
We analyze \emph{$\sH$-consistency bounds}
\citep{Awasthi2022Hconsistency,mao2023cross}, which relate the estimation error
of the target loss to that of the surrogate. Such bounds typically take the
form
\begin{equation}
  \label{eq:hcons}
  \sE_{\ell_{\tar}}(h) - \sE^*_{\ell_{\tar}}(\sH) + \sM_{\ell_{\tar}}(\sH)
  \leq
  \Gamma\paren*{\sE_{\ell_{\sur}}(h) - \sE^*_{\ell_{\sur}}(\sH) +
    \sM_{\ell_{\sur}}(\sH)},
\end{equation}
where $\Gamma$ with $\Gamma(0) = 0$ is an increasing concave function (e.g.,
$t \mapsto t$ or $t \mapsto \sqrt{t}$). The term $\sM_\ell(\sH)$ is the
\emph{minimizability gap}, defined as
\begin{equation*}
  \sM_\ell(\sH) = \sE^*_\ell(\sH) - \E_x\bracket*{\inf_{h \in \sH} \E_y \bracket*{\ell(h, x, y) \mid x}}.
\end{equation*}
This quantity measures the discrepancy between the best possible expected loss
within $\sH$ and the expected pointwise infimum. It is upper bounded by, yet
generally finer than, the standard approximation error
$\sE^*_\ell(\sH) - \sE^*_\ell(\sH_{\rm{all}})$, where $\sH_{\rm{all}}$ denotes
the family of all measurable functions \citep{MaoMohriZhong2024}.  We denote the
excess target and surrogate errors by
\begin{equation*}
  \Delta \sR = \sE_{\ell_{\tar}}(h) - \sE^*_{\ell_{\tar}}(\sH_{\rm{all}})
  \quad \text{and} \quad
  \Delta \sL = \sE_{\ell_{\sur}}(h) - \sE^*_{\ell_{\sur}}(\sH_{\rm{all}}),
\end{equation*}
respectively.  When $\sH = \sH_{\rm{all}}$, the minimizability gap reduces to
the approximation error. Consequently, in this case, an $\sH$-consistency bound
implies the standard excess error bound $\Delta \sR \leq \Gamma (\Delta
\sL)$. Thus, $\sH$-consistency bound is a strictly stronger guarantee than
standard Bayes-consistency.

\paragraph{Conditional errors and regrets.}
Our proofs work pointwise in $x$. For a loss function $\ell$, we define the
\emph{conditional error} of a hypothesis $h \in \sH$ at a point $x \in \sX$ as
\begin{equation*}
  \sC_{\ell}(h, x) = \sum_{y \in \sY} p(y \mid x) \, \ell(h, x, y),
\end{equation*}
where $p(y \mid x) = \sD(Y = y \mid X = x)$ is the conditional probability of
$y$ given $x$.  The \emph{best-in-class conditional error} is defined as
$\sC_{\ell}^*(\sH, x) = \inf_{h\in \sH}\sC_{\ell}(h, x)$, and the
\emph{conditional regret} is the difference between the conditional error and
the best-in-class conditional error:
\begin{equation*}
  \Delta\sC_{\ell,\sH}(h,x) = \sC_{\ell}(h,x) - \sC_{\ell}^*(\sH, x).
\end{equation*}
The generalization error can be expressed as the expectation of the conditional
error: $\sE_{\ell}(h) = \E_{x}\bracket*{\sC_{\ell}(h, x)}$. In particular,
$\sM_\ell(\sH) = \sE^*_\ell(\sH) - \E_x\bracket*{\sC^*_\ell(\sH, x)}$, so that
$\sE_\ell(h) - \sE^*_\ell(\sH) + \sM_\ell(\sH) = \E_x\bracket*{\Delta\sC_{\ell,\sH}(h,x)}$.

%%%%%%%%%%%%%%%%%%%%%%%%%%%%%%%%%%%%%%%%%%%%%%%%%%%%%%%%%%%%%%%%%%%%%%%%%%%%%%%
\section{Binary Classification}
\label{sec:binary}
%%%%%%%%%%%%%%%%%%%%%%%%%%%%%%%%%%%%%%%%%%%%%%%%%%%%%%%%%%%%%%%%%%%%%%%%%%%%%%%

We first consider the binary classification setting where $\ell_{\tar}$ is the
binary zero-one loss, defined by
$\ell_{0-1} (h, x, y) = 1_{\sign(h(x)) \neq y}$.  When the target is the binary
zero-one loss, we define the \emph{$\sH$-estimation error transformation} $\sT$
as the function satisfying the following tight lower bound for all $h \in \sH$:
\begin{equation*}
  \sT \paren*{\sE_{\ell_{0-1}}(h) - \sE^*_{\ell_{0-1}}(\sH) +
    \sM_{\ell_{0-1}}(\sH)} \leq \sE_{\ell_{\sur}}(h)-\sE^*_{\ell_{\sur}}(\sH) +
  \sM_{\ell_{\sur}}(\sH).
\end{equation*}
Tightness implies that for any $t \in [0, 1]$, there exists a distribution and a
hypothesis such that
$\sE_{\ell_{0-1}}(h) - \sE^*_{\ell_{0-1}}(\sH) + \sM_{\ell_{0-1}}(\sH) = t$ and
$\sE_{\ell_{\sur}}(h)-\sE^*_{\ell_{\sur}}(\sH) + \sM_{\ell_{\sur}}(\sH) =
\sT(t)$. Explicit forms of $\sT$ have been characterized for binary margin-based
losses \citep{Awasthi2022Hconsistency,MaoMohriZhong2024}.

\subsection{Smooth Surrogates with Linear Cores}
\label{sec:binary-def}

Let $\Phi \colon \Rset \to \Rset_{+}$ be a differentiable convex function with
$\Phi'(0) > 0$. For instance, one may take $\Phi(u) = \log(1 + e^{u})$,
corresponding to the logistic loss, or $\Phi(u) = e^{u}$, corresponding to the
exponential loss.  We define the smooth \emph{linear-core (LC) surrogate} loss
as the margin-based loss function
$\ell_{\ov \Phi} \colon (h, x, y) \mapsto \ov \Phi(y h(x))$, where $\ov \Phi$
is given by
\begin{equation}
  \label{eq:Phi}
  \ov \Phi(u) =
  \begin{cases}
    - u + 1 + \dfrac{\Phi(0)}{\Phi'(0)}, & -1 \leq u \leq 1, \\[8pt]
    \dfrac{\Phi(1 - u)}{\Phi'(0)}, & u > 1, \\[10pt]
    \dfrac{\Phi(-1 - u)}{\Phi'(0)} + 2, & u < -1.
  \end{cases}
\end{equation}
Figure~\ref{fig:extensions} gives several illustrations of $\ov \Phi$.
We also consider a \emph{one-sided} linear-core surrogate, obtained by smoothing
only the right outer branch and keeping the left side linear. We define
$\wt \Phi$ as:
\begin{equation}
  \label{eq:Phi_one}
  \wt \Phi(u) =
  \begin{cases}
    -\,u + 1 + \dfrac{\Phi(0)}{\Phi'(0)}, & u \leq 1,\\[8pt]
    \dfrac{\Phi(1 - u)}{\Phi'(0)}, & u > 1.
  \end{cases}
\end{equation}
Note that any constant can be added to the loss function without affecting the
minimization.

\begin{figure}[t]
  \centering
  \includegraphics[width=0.44\linewidth]{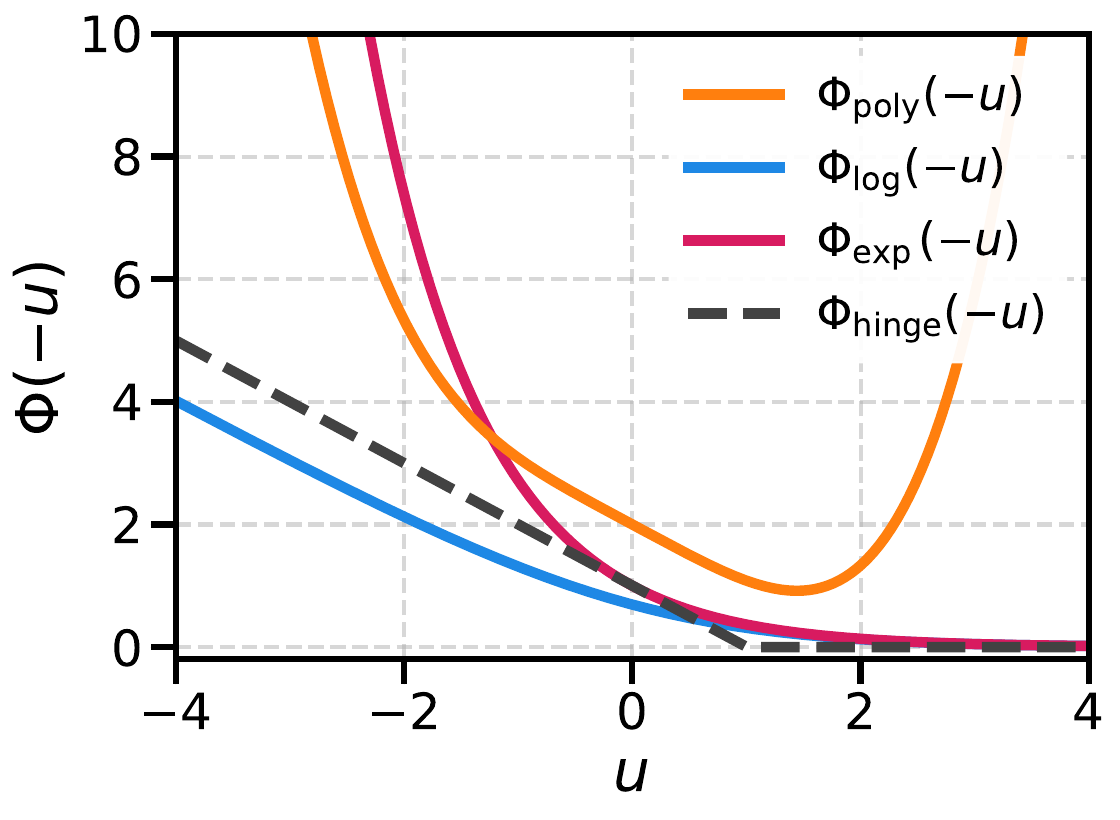}
  \hspace{0.04\linewidth}
  \includegraphics[width=0.44\linewidth]{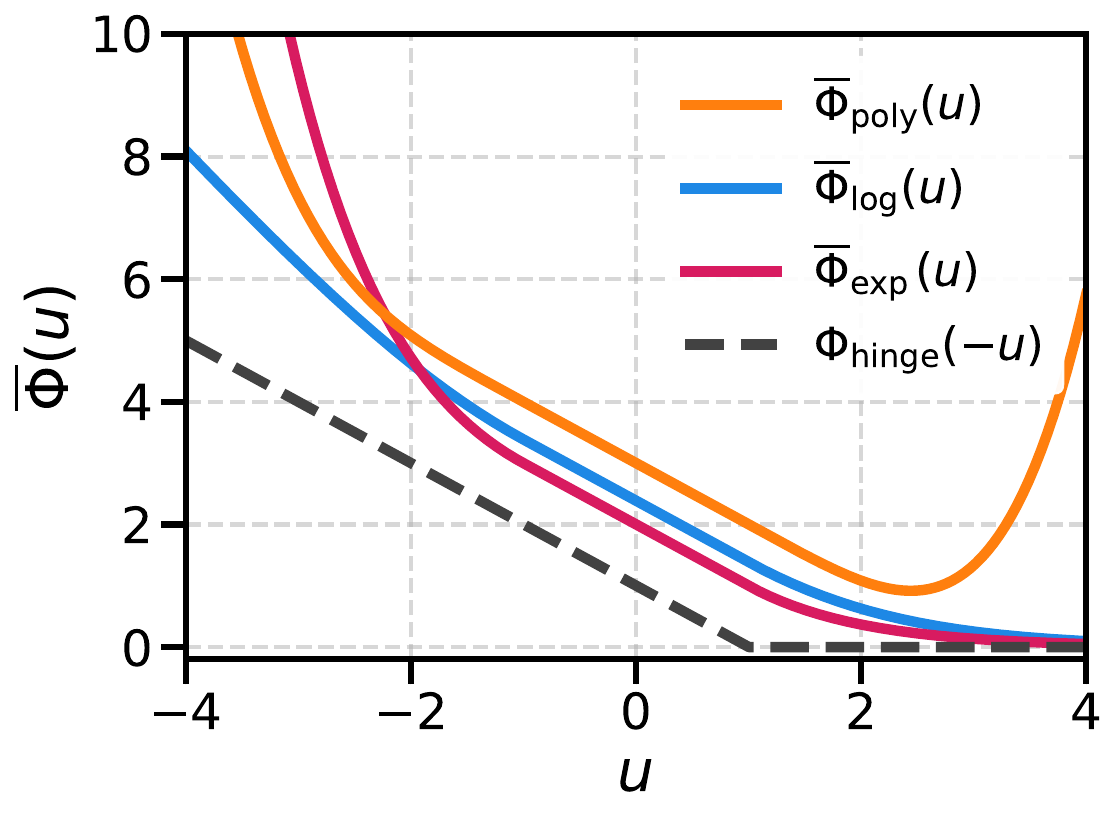}
  \caption{Left: Standard surrogates. Right: LC surrogates $\ov \Phi$.}
  \label{fig:extensions}
\end{figure}

The following proposition establishes the key analytical properties of the
binary Linear-Core surrogates. These properties are inherited by the
multi-class and structured surrogates
(Propositions~\ref{prop:properties-multi} and~\ref{prop:properties-struct}).

\begin{proposition}[Convexity and Smoothness]
  \label{prop:properties}
  Let $\phi \in \curl*{\ov \Phi, \wt \Phi}$.  Then $\phi$ is globally convex
  and continuously differentiable ($C^1$). Furthermore, if the base $\Phi$ is
  twice continuously differentiable with $\Phi''(0)=0$, then $\phi$ is twice
  continuously differentiable ($C^2$).
\end{proposition}

\begin{proof}
  \textbf{1. Symmetric surrogate $\ov \Phi$.}
  \emph{Convexity:} Each branch of $\ov \Phi$ in \eqref{eq:Phi} is convex on its
  own interval:
  \begin{itemize}
  \item On $(-1, 1)$, $\ov \Phi$ is linear, hence convex.
  \item On $(1, \infty)$, $u \mapsto 1 - u$ is affine. Since $\Phi$ is convex,
    $u \mapsto \Phi(1 - u)$ is convex. Positive scaling by $1 / \Phi'(0)$
    preserves convexity.
  \item On $(-\infty, -1)$, similarly, the composition with the affine map
    $u \mapsto -1 - u$ is convex.
  \end{itemize}
  It remains to check the junctions $u = \pm 1$. A continuous piecewise $C^1$
  function is convex if the derivative is non-decreasing, which requires
  $\ov \Phi'_-(u_0) \leq \ov \Phi'_+(u_0)$ at any junction $u_0$.  Here,
  $\ov \Phi'(u) = -1$ on $(-1, 1)$.  For $u > 1$,
  $\ov \Phi'(u) = - \frac{\Phi'(1 - u)}{\Phi'(0)}$. As $u \to 1^+$,
  $1-u \to 0^-$, so $\ov \Phi'_+(1) = - \frac{\Phi'(0)}{\Phi'(0)} = -1$. This
  matches $\ov \Phi'_-(1) = -1$.  For $u < -1$,
  $\ov \Phi'(u) = - \frac{\Phi'(-1 - u)}{\Phi'(0)}$. As $u \to -1^-$,
  $-1-u \to 0^+$, so $\ov \Phi'_-(-1) = - \frac{\Phi'(0)}{\Phi'(0)} = -1$. This
  matches $\ov \Phi'_+(-1) = -1$.  Thus, the derivative is continuous everywhere
  and $\ov \Phi$ is convex.

  \emph{Smoothness ($C^1$):} As shown above, the one-sided derivatives match at
  $u = \pm 1$. Since $\Phi'$ is continuous on the outer intervals, $\ov \Phi'$
  is continuous everywhere. Thus $\ov \Phi \in C^1(\Rset)$.

  \emph{Smoothness ($C^2$):} For $u > 1$,
  $\ov \Phi''(u) = \frac{\Phi''(1 - u)}{\Phi'(0)}$. For $u < -1$,
  $\ov \Phi''(u) = \frac{\Phi''(-1 - u)}{\Phi'(0)}$. On $(-1, 1)$,
  $\ov \Phi''(u) = 0$.  Continuity at $u=1$ requires
  $\lim_{u \to 1^+} \ov \Phi''(u) = 0$, which implies
  $\lim_{z \to 0^-} \Phi''(z) = 0$.  Similarly at $u=-1$, we need
  $\lim_{z \to 0^+} \Phi''(z) = 0$.  Thus, $\ov \Phi \in C^2(\Rset)$ if and only
  if $\lim_{z \to 0} \Phi''(z) = 0$. In particular, if $\Phi \in C^2$ and
  $\Phi''(0)=0$, this holds.

  \textbf{2. One-sided surrogate $\wt \Phi$.}
  \emph{Convexity:} On $(-\infty, 1]$, $\wt \Phi$ is linear (convex). On
  $(1, \infty)$, it matches $\ov \Phi$ (convex). At $u=1$, the derivatives match
  at $-1$. Thus $\wt \Phi$ is convex.

  \emph{Smoothness ($C^1$ and $C^2$):} Matching derivatives at $u=1$ implies
  $\wt \Phi \in C^1(\Rset)$.  For $C^2$, we require
  $\lim_{u \to 1^+} \wt \Phi''(u) = 0$, which implies
  $\lim_{z \to 0^-} \Phi''(z) = 0$.
\end{proof}

The condition $\Phi''(0) = 0$ is not vacuous. For example, take
$\Phi(u) = a\, u + \frac{1}{12} u^{4} + K$ with $a > 0$. Then $\Phi''(0) = 0$,
and by Proposition~\ref{prop:properties}, the corresponding linear-core
surrogate is twice continuously differentiable on $\Rset$. In contrast, for
common choices such as the logistic loss $\Phi(u) = \log(1+e^{u})$ or the
exponential loss $\Phi(u) = e^{u}$, one has $\Phi''(0) > 0$, so $\ov \Phi$ is
$C^1$ but not $C^2$ at the hinge points $u = \pm 1$.

\subsection{Linear \texorpdfstring{$\sH$}{H}-Consistency Bound}
\label{sec:binary-bound}

We call a hypothesis set $\sH$ complete if, for every $x \in \sX$,
$\{ h(x) \colon h \in \sH \} = \Rset$. Since $\ov \Phi$ is convex and
differentiable at zero and satisfies the inequality $\ov \Phi'(0) = -1 < 0$, by
\citet[Theorem~4.1]{MaoMohriZhong2024}, for complete hypothesis sets, the tight
transformation $\sT$ takes the following form.

\begin{theorem}
  \label{thm:T-binary}
  Let $\sH$ be a complete hypothesis set.  The transformation $\sT$ can be
  expressed as follows:
  \begin{equation*}
    \forall t \in [0, 1], \quad \sT(t) = \ov \Phi(0) - \inf_{u \in \Rset} \paren*{\tfrac{1 - t}{2} \ov \Phi(-u) + \tfrac{1 + t}{2} \ov \Phi(u)}.
  \end{equation*}
\end{theorem}

\begin{proof}
  Since $\ov \Phi$ is convex and differentiable at zero and satisfies
  $\ov \Phi'(0) = -1 < 0$, the result follows directly from
  \citet[Theorem~4.1]{MaoMohriZhong2024}.
\end{proof}

The following result shows that the transformation is bounded below by a linear
function of $t$.

\begin{lemma}
  \label{lemma:lin-binary}
  For all $t \in [0, 1]$, $\sT(t) \geq t$ and $\sT(0) = 0$.
\end{lemma}

\begin{proof}
  By Theorem~\ref{thm:T-binary},
  \begin{equation*}
    \sT(t) = \ov \Phi(0) - \inf_{u \in \Rset} 
    \paren*{\tfrac{1 - t}{2} \, \ov \Phi(-u) + \tfrac{1 + t}{2} \, \ov \Phi(u)}.
  \end{equation*}

  \textbf{Step 1. Lower bound $\sT(t) \geq t$.}  On $[-1, 1]$, $\ov \Phi$ is
  linear:
  \[
    \ov \Phi(u) = -u + 1 + \frac{\Phi(0)}{\Phi'(0)}, \qquad \ov \Phi(-u) = u + 1
    + \frac{\Phi(0)}{\Phi'(0)}.
  \]
  Thus for any $u \in [-1, 1]$,
  \begin{equation*}
    \frac{1 - t}{2} \, \ov \Phi(-u) + \frac{1 + t}{2} \, \ov \Phi(u)
    = \frac{1 - t}{2} \paren*{u + 1 + \frac{\Phi(0)}{\Phi'(0)}}
    + \frac{1 + t}{2} \paren*{-u + 1 + \frac{\Phi(0)}{\Phi'(0)}}
    = \paren*{\frac{\Phi(0)}{\Phi'(0)} + 1} - t u.
  \end{equation*}
  For $t \in [0, 1]$, this expression is minimized over $u \in [-1, 1]$ at
  $u = 1$, giving
  \begin{equation*}
    \inf_{u \in \Rset} 
    \paren*{\tfrac{1 - t}{2} \, \ov \Phi(-u) + \tfrac{1 + t}{2} \, \ov \Phi(u)}
    \leq \paren*{\frac{\Phi(0)}{\Phi'(0)} + 1} - t.
  \end{equation*}
  Since $\ov \Phi(0) = \frac{\Phi(0)}{\Phi'(0)} + 1$ (because $0 \in [-1, 1]$),
  we obtain
  \begin{equation*}
    \sT(t) \geq 
    \paren*{\frac{\Phi(0)}{\Phi'(0)} + 1}
    - \paren*{\paren*{\frac{\Phi(0)}{\Phi'(0)} + 1} - t}
    = t.
  \end{equation*}

  \textbf{Step 2. Exact value at $t = 0$.}  When $t = 0$, we have
  \begin{equation*}
    \sT(0) = \ov \Phi(0) - \inf_{u \in \Rset} 
    \tfrac{1}{2} \paren*{\ov \Phi(-u) + \ov \Phi(u)}.
  \end{equation*}
  If $u \in [-1, 1]$, then
  \begin{equation*}
    \tfrac{1}{2} \paren*{\ov \Phi(-u) + \ov \Phi(u)}
    = \tfrac{1}{2} \paren*{u + 1 + \tfrac{\Phi(0)}{\Phi'(0)} 
      + \paren*{-u + 1 + \tfrac{\Phi(0)}{\Phi'(0)}}}
    = \frac{\Phi(0)}{\Phi'(0)} + 1.
  \end{equation*}
  Thus every $u \in [-1, 1]$ attains the value $\frac{\Phi(0)}{\Phi'(0)} + 1$.
  For $u > 1$ or $u < -1$, one checks from the outer branches of $\ov \Phi$ and
  convexity that
  \begin{equation*}
    \tfrac{1}{2} \paren*{\ov \Phi(-u) + \ov \Phi(u)} 
    \geq \frac{\Phi(0)}{\Phi'(0)} + 1,
  \end{equation*}
  with equality only at the boundary $u = \pm 1$.  Therefore,
  \begin{equation*}
    \inf_{u \in \Rset} \tfrac{1}{2} \paren*{\ov \Phi(-u) + \ov \Phi(u)} 
    = \frac{\Phi(0)}{\Phi'(0)} + 1,
  \end{equation*}
  and the set of minimizers is precisely $[-1, 1]$.  Since
  $\ov \Phi(0) = \frac{\Phi(0)}{\Phi'(0)} + 1$, it follows that $\sT(0) = 0$.
  Combining the two steps proves the claim.
\end{proof}

By Lemma~\ref{lemma:lin-binary} together with
\citet[Theorem~4]{Awasthi2022Hconsistency}, we obtain a linear $\sH$-consistency
bound for the surrogate losses $\ell_{\ov \Phi}$.

\begin{theorem}[Linear $\sH$-consistency bound]
  \label{thm:bound-binary}
  Let $\sH$ be a complete hypothesis set.  Then, for all $h \in \sH$,
  \begin{equation*}
    \sE_{\ell_{0-1}}(h) - \sE^*_{\ell_{0-1}}(\sH) + \sM_{\ell_{0-1}}(\sH)
    \leq
    \sE_{\ell_{\ov \Phi}}(h) - \sE^*_{\ell_{\ov \Phi}}(\sH) + \sM_{\ell_{\ov \Phi}}(\sH).
  \end{equation*}
\end{theorem}

\begin{proof}
  By \citet[Theorem~4.1]{MaoMohriZhong2024}, we have
  \begin{equation*}
    \forall h \in \sH, \quad
    \sT \paren*{\sE_{\ell_{0-1}}(h) - \sE^*_{\ell_{0-1}}(\sH) + \sM_{\ell_{0-1}}(\sH)}
    \leq
    \sE_{\ell_{\ov \Phi}}(h) - \sE^*_{\ell_{\ov \Phi}}(\sH) + \sM_{\ell_{\ov \Phi}}(\sH).
  \end{equation*}
  By Lemma~\ref{lemma:lin-binary}, since $\sT(t) \geq t$ for all $t \in [0, 1]$,
  it follows that
  \begin{equation*}
    \forall h \in \sH, \quad
    \sE_{\ell_{0-1}}(h) - \sE^*_{\ell_{0-1}}(\sH) + \sM_{\ell_{0-1}}(\sH)
    \leq
    \sE_{\ell_{\ov \Phi}}(h) - \sE^*_{\ell_{\ov \Phi}}(\sH) + \sM_{\ell_{\ov \Phi}}(\sH).
  \end{equation*}
  This completes the proof.
\end{proof}

\paragraph{Intuition.}
The proof relies on the $\sH$-consistency framework of
\citet{Awasthi2022Hconsistency}, and its key step is the analysis of the
transformation $\sT(t)$, which relates the surrogate estimation error to the
target error. Since $\ov{\Phi}$ is linear with slope $-1$ on the interval
$[-1, 1]$, $\sT$ satisfies the lower bound $\sT(t) \geq t$ for all
$t \in [0, 1]$ (Lemma~\ref{lemma:lin-binary}). This linear lower bound directly
implies the linear consistency rate $\Delta \sR \le \Delta \sL$ when
$\sH = \sH_{\rm{all}}$.  Unlike smooth losses where the gradient vanishes at
the origin (causing ``flat'' landscapes and slow $\sqrt{t}$ transfer), the
Linear-Core surrogate maintains a non-zero gradient near the decision boundary.
This forces the surrogate estimation error to scale linearly with the target
error.

Combining Proposition~\ref{prop:properties} with Theorem~\ref{thm:bound-binary},
we obtain convex and smooth (even twice continuously differentiable) surrogate
losses with linear $\sH$-consistency bounds. Note that this does not conflict
with \citet[Theorem~4.2]{MaoMohriZhong2024}, since $\ov \Phi''(0) = 0$ (see
Section~\ref{sec:related} for a detailed discussion).

\subsection{One-Sided Smoothing}
\label{sec:binary-one}

We next analyze the consistency of the one-sided linear-core surrogate
$\wt \Phi$ defined in Eq.~\eqref{eq:Phi_one}. By
Proposition~\ref{prop:properties}, $\wt \Phi$ is convex and smooth. The linear
lower bound on the transformation also holds for this variant.

\begin{lemma}[Linear bound for one-sided smoothing]
  \label{lemma:lin-binary-one}
  For $t \in [0, 1]$, define
  \begin{equation*}
    \sT_{\mathrm{one}}(t) = \wt \Phi(0) - \inf_{u \in \Rset} \paren*{\tfrac{1 - t}{2}\,\wt \Phi(-u) + \tfrac{1 + t}{2}\,\wt \Phi(u)}.
  \end{equation*}
  Then $\sT_{\mathrm{one}}(t) \geq t$ and $\sT_{\mathrm{one}}(0) = 0$.
\end{lemma}

The proof is essentially identical to that of Lemma~\ref{lemma:lin-binary},
since both arguments rely only on the linear core of $\wt \Phi$ over $[-1, 1]$.
This shows that the fundamental linear lower bound carries over unchanged to the
one-sided case.

\begin{proof}
  For $u \in [-1, 1]$, $\wt \Phi(u) = - u + 1 + \Phi(0)/\Phi'(0)$ and
  $\wt \Phi(-u) = u + 1 + \Phi(0)/\Phi'(0)$. Thus
  \begin{equation*}
    \tfrac{1 - t}{2}\, \wt \Phi(-u) + \tfrac{1 + t}{2}\, \wt \Phi(u)
    = \Bigl( \dfrac{\Phi(0)}{\Phi'(0)} + 1 \Bigr) - t\, u.
  \end{equation*}
  Minimizing over $u \in [-1, 1]$ gives
  \[
    \inf_{u} \paren*{\tfrac{1 - t}{2}\, \wt \Phi(-u) + \tfrac{1 + t}{2}\, \wt
      \Phi(u)} \leq \dfrac{\Phi(0)}{\Phi'(0)} + 1 - t,
  \]
  and since $\wt \Phi(0) = \dfrac{\Phi(0)}{\Phi'(0)} + 1$, we obtain
  $\sT_{\mathrm{one}}(t) \geq t$. At $t = 0$, the same calculation shows
  $\sT_{\mathrm{one}}(0) = 0$.
\end{proof}

Furthermore, Theorem~\ref{thm:bound-binary} also remains valid without
modification, since its proof relies only on Lemma~\ref{lemma:lin-binary}, which
we have extended to the one-sided smoothing case in
Lemma~\ref{lemma:lin-binary-one}.

\begin{corollary}[Linear $\sH$-consistency bound for one-sided smoothing]
  \label{cor:bound-binary-one}
  Let $\sH$ be a complete hypothesis set.  Then for $\wt \Phi$, the following
  linear $\sH$-consistency bound holds:
  \begin{equation*}
    \forall h \in \sH, \quad
    \sE_{\ell_{0-1}}(h) - \sE^*_{\ell_{0-1}}(\sH)
    + \sM_{\ell_{0-1}}(\sH)
    \leq
    \sE_{\ell_{\wt \Phi}}(h) - \sE^*_{\ell_{\wt \Phi}}(\sH) + \sM_{\ell_{\wt \Phi}}(\sH).
  \end{equation*}
\end{corollary}

\begin{proof}
  Since $\wt \Phi$ is convex and differentiable at zero and satisfies the
  inequality $\wt \Phi'(0) = -1 < 0$, by \citet[Theorem~4.1]{MaoMohriZhong2024},
  for complete hypothesis sets, the transformation $\sT$ is equal to
  $\sT_{\mathrm{one}}$:
  \begin{equation*}
    \forall h \in \sH, \quad
    \sT_{\mathrm{one}} \paren*{\sE_{\ell_{0-1}}(h) - \sE^*_{\ell_{0-1}}(\sH)
      + \sM_{\ell_{0-1}}(\sH)}
    \leq
    \sE_{\ell_{\wt \Phi}}(h) - \sE^*_{\ell_{\wt \Phi}}(\sH) + \sM_{\ell_{\wt \Phi}}(\sH).
  \end{equation*}
  By Lemma~\ref{lemma:lin-binary-one}, since $\sT_{\mathrm{one}}(t) \geq t$ for
  all $t \in [0, 1]$, it follows that
  \begin{equation*}
    \forall h \in \sH, \quad
    \sE_{\ell_{0-1}}(h) - \sE^*_{\ell_{0-1}}(\sH)
    + \sM_{\ell_{0-1}}(\sH)
    \leq
    \sE_{\ell_{\wt \Phi}}(h) - \sE^*_{\ell_{\wt \Phi}}(\sH) + \sM_{\ell_{\wt \Phi}}(\sH).
  \end{equation*}
  This completes the proof.
\end{proof}

Intuitively, $\wt \Phi$ inherits all desirable properties of $\ov \Phi$
(convexity, smoothness, and linear $\sH$-consistency bounds) while smoothing
only one side.  This is attractive when one prefers to soften the penalty for
large positive margins while keeping the negative side linear, e.g., when the
objective tolerates large positive scores but requires sharper control on the
negative side.

\subsection{Empirical Validation of Convergence Rates}
\label{sec:empirical_validation}

To illustrate the theoretical distinction between the linear $\sH$-consistency
of our Linear-Core Surrogates and the slower consistency of standard smooth
losses, we analyze a canonical \emph{biased coin} problem
\citep{bartlett2006convexity}. We consider a binary classification task where
the label probability is $\eta = 1/2 + \delta$, with $\delta > 0$ representing
the margin.

We compute the exact excess surrogate error $\Delta \sL$ and excess target error
$\Delta \sR$ analytically across a range of margins
$\delta \in [10^{-4}, 10^{-1}]$. This setup removes finite-sample optimization
noise and isolates the asymptotic convergence behavior of the loss
functions. Figure~\ref{fig:rates} reports the results.

\begin{figure}[t]
  \centering
  \includegraphics[width=0.5\linewidth]{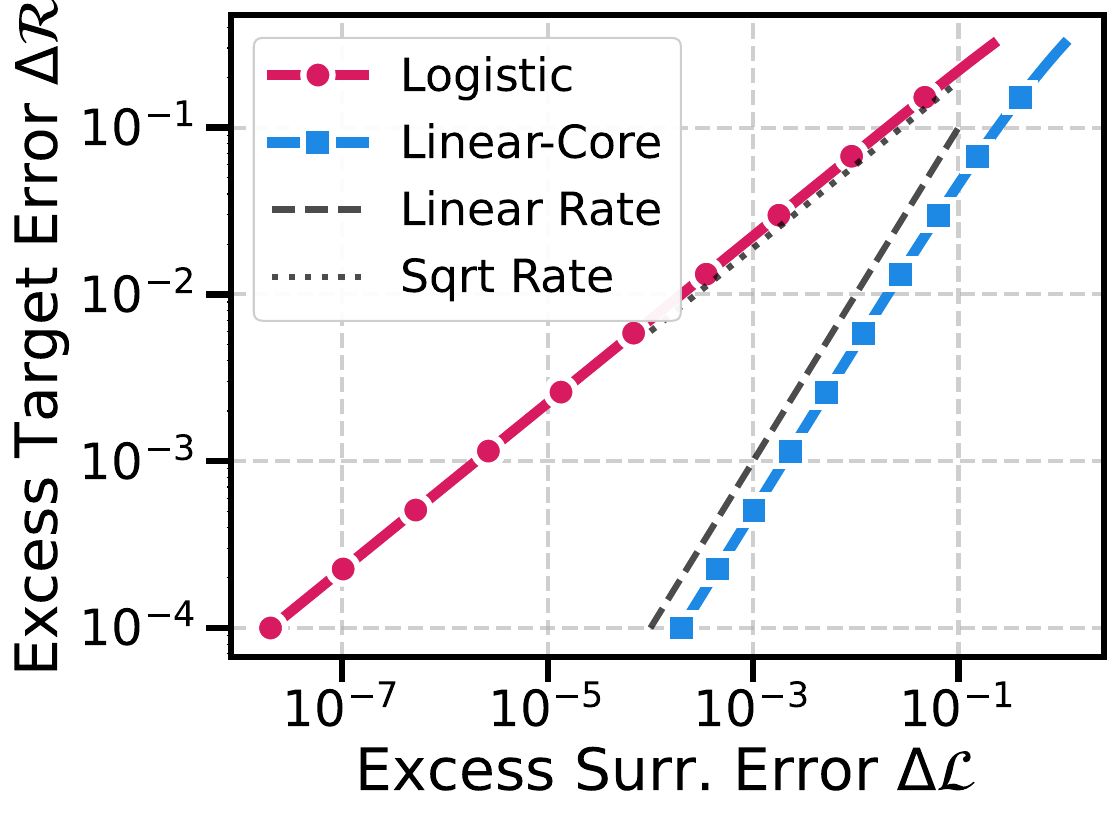}
  \caption{Rates: \textcolor{DarkBlue}{LC} vs. \textcolor{DarkRed}{Logistic}.
    Excess target error $\Delta \sR$ versus excess surrogate error $\Delta \sL$
    on the biased coin problem (log-log scale).}
  \label{fig:rates}
\end{figure}

The log-log plot confirms that the Linear-Core Surrogate (blue) maintains a
strict linear relationship (slope $\approx 1$) where
$\Delta \sR = O(\Delta \sL)$. In contrast, the standard Logistic loss (red)
exhibits the slower square-root relationship (slope $\approx 0.5$)
characteristic of losses with vanishing curvature, where
$\Delta \sR = O(\sqrt{\Delta \sL})$.  This confirms that for hard problems
(small $\delta$), minimizing the Linear-Core Surrogate translates to target
error reduction significantly faster than minimizing the logistic loss.
Section~\ref{sec:stability} further considers a generalized family
parameterized by $\tau > 0$, showing that the linear rates are robust to this
choice.

\subsection{Stability Analysis}
\label{sec:stability}

To verify that the fast linear rates are not an artifact of the specific
interval $[-1, 1]$, we consider a generalized family of Linear-Core surrogates
parameterized by a threshold $\tau > 0$. We define the generalized surrogate
$\ov{\Phi}_\tau$ by stitching the linear core on $[-\tau, \tau]$ to the smooth
tail:
\begin{equation}
  \label{eq:generalized_phi}
  \ov \Phi_\tau(u) =
  \begin{cases}
    - u + \tau + \dfrac{\Phi(0)}{\Phi'(0)}, & -\tau \leq u \leq \tau, \\[8pt]
    \dfrac{\Phi(\tau - u)}{\Phi'(0)}, & u > \tau, \\[10pt]
    \dfrac{\Phi(-\tau - u)}{\Phi'(0)} + 2\tau, & u < -\tau.
  \end{cases}
\end{equation}
The proof of Proposition~\ref{prop:properties} applies verbatim, with the
junctions at $\pm \tau$ instead of $\pm 1$, so $\ov{\Phi}_\tau$ inherits the
same convexity and smoothness properties. Furthermore, the linear
$\sH$-consistency bound (Theorem~\ref{thm:bound-binary}) extends naturally to
$\ov{\Phi}_\tau$. By the argument of Lemma~\ref{lemma:lin-binary}, with the
minimization now over $u \in [-\tau, \tau]$, the transformation $\sT$ maintains
a linear lower bound $\sT(t) \geq \tau \, t$ with $\sT(0) = 0$, that is,
$\Delta \sR \leq \Delta \sL / \tau$, preserving the fast $O(\Delta \sL)$
convergence rate. Note that the constant $1/\tau$ blows up as $\tau \to 0$,
consistent with the vanishing-core experiment below.

We first analyzed the convergence rates for robust thresholds
$\tau \in \{0.1, 0.5, 1.0, 2.0, 5.0\}$. As shown in
Figure~\ref{fig:stability_thresholds}, the linear convergence rate is robust to
the choice of $\tau$ in this regime. All Linear-Core variants maintain a slope
of $1$ (implying $\Delta \sR = O(\Delta \sL)$), standing in sharp contrast to
the Logistic loss, which degrades to a slope of $1/2$ (implying
$\Delta \sR = O(\sqrt{\Delta \sL})$). This confirms that the fast rate is driven
by the linear segment at the origin, where the slope is non-zero and the
curvature vanishes.

However, we also investigated the limit case where the linear core vanishes
($\tau \to 0$). We tested microscopic thresholds
$\tau \in \{10^{-1}, \dots, 10^{-5}\}$. As shown in
Figure~\ref{fig:vanishing_core}, as $\tau$ approaches zero, the surrogate
$\ov{\Phi}_\tau$ effectively reverts to a standard smooth loss
function. Consequently, the acceleration vanishes: the curves for the smallest
thresholds (e.g., $\tau = 10^{-5}$) align with the Logistic baseline, exhibiting
the slower square-root convergence rate ($y \propto \sqrt{x}$). This
demonstrates that the acceleration is strictly dependent on the presence of a
non-negligible linear component; when this component is removed, the fast rate
is lost.

\begin{figure}[t]
  \centering
  \begin{minipage}[t]{0.48\linewidth}
    \centering
    \includegraphics[width=\linewidth]{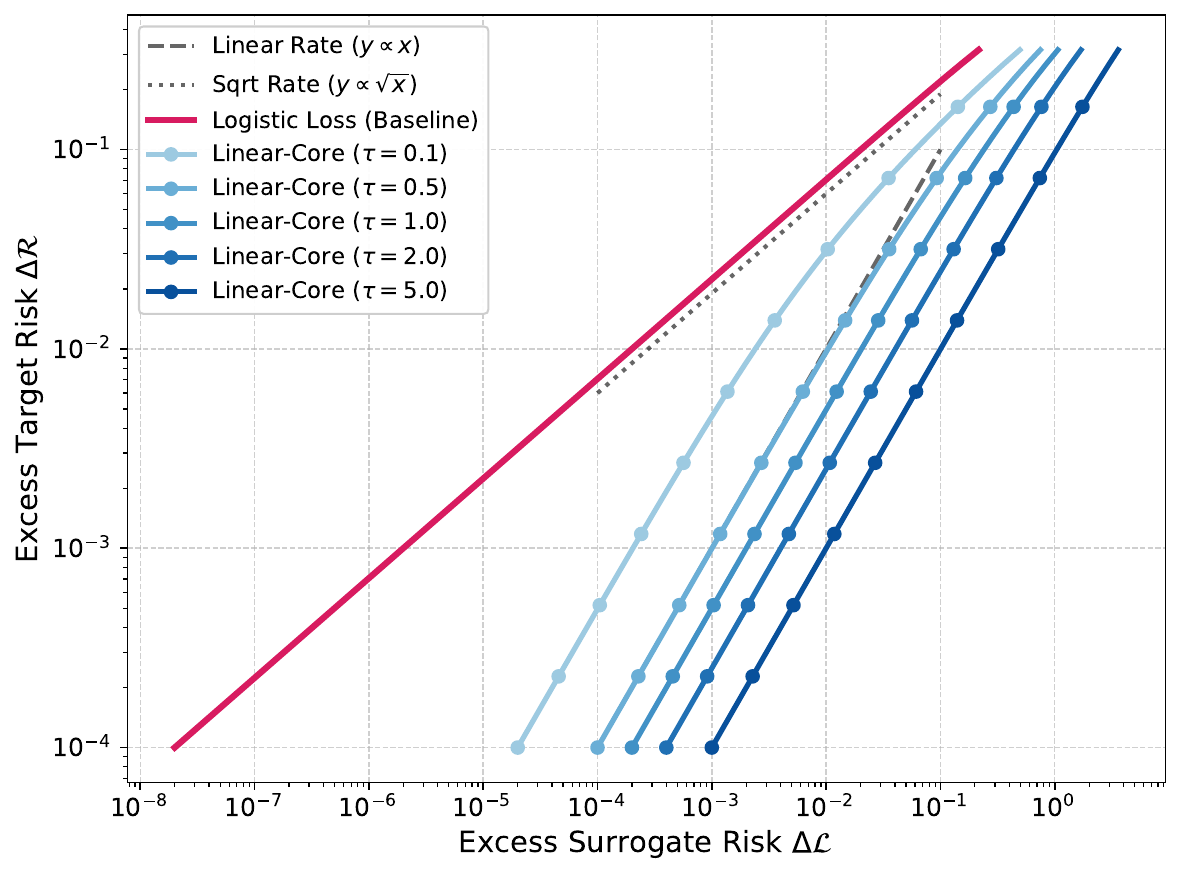}
    \caption{Stability of Convergence Rates across Core Thresholds.  We compare
      the excess target error $\Delta \mathcal{R}$ against the excess surrogate
      error $\Delta \mathcal{L}$ for the generalized Linear-Core surrogate
      $\ov{\Phi}_\tau$ with varying threshold widths $\tau$.  Regardless of
      whether the linear core is narrow ($\tau=0.1$) or wide ($\tau=5.0$), all
      \textcolor{DarkBlue}{Linear-Core variants (Blue lines)} maintain a strict
      linear convergence rate parallel to the $y \propto x$ asymptote.  In
      contrast, the \textcolor{DarkRed}{Logistic loss (Red)} exhibits the slower
      square-root rate ($y \propto \sqrt{x}$).}
    \label{fig:stability_thresholds}
  \end{minipage}\hfill
  \begin{minipage}[t]{0.48\linewidth}
    \centering
    \includegraphics[width=\linewidth]{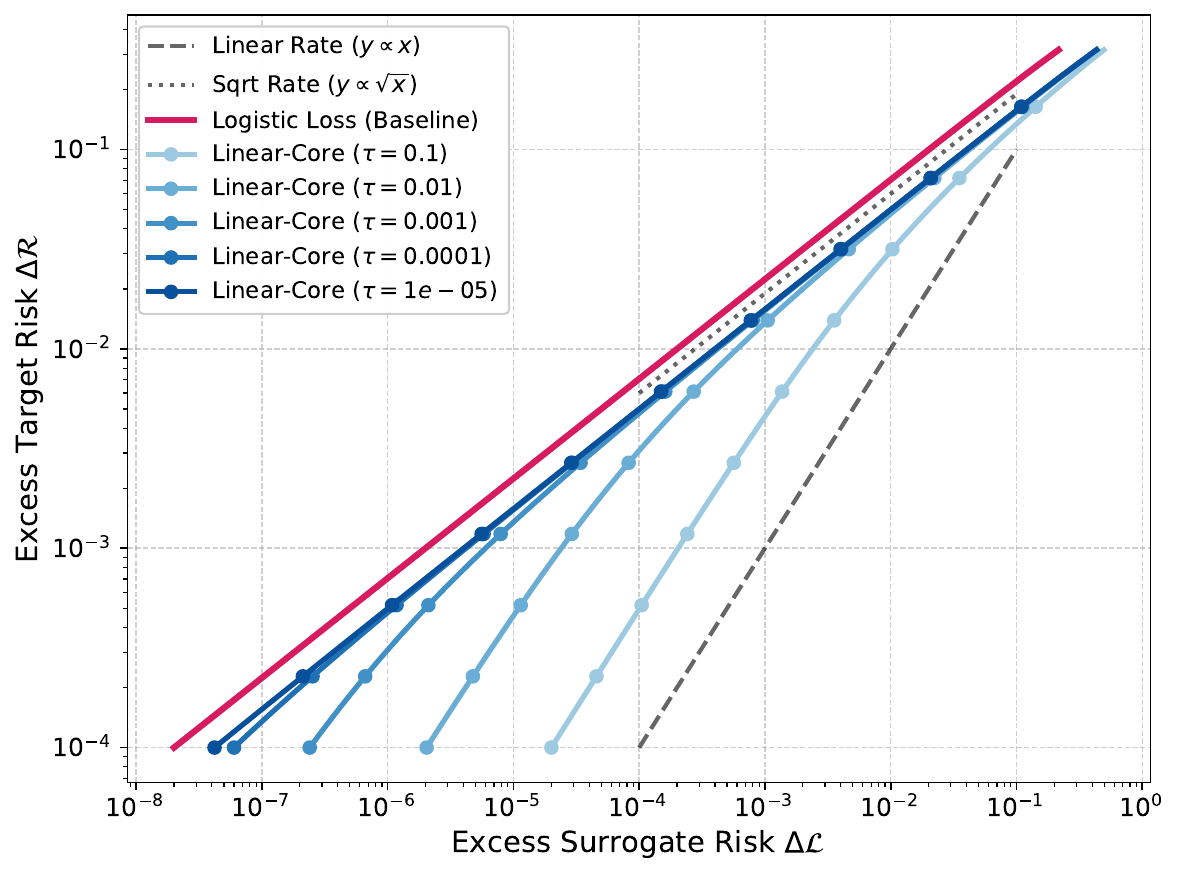}
    \caption{Degradation of Convergence Rates as Core Vanishes ($\tau \to 0$).
      We analyze the convergence behavior for decreasing thresholds
      $\tau \in \{10^{-1}, \dots, 10^{-5}\}$.  As the linear core shrinks, the
      \textcolor{DarkBlue}{Linear-Core variants} (darker blue lines) gradually
      lose their linear acceleration.  For the smallest threshold
      ($\tau = 10^{-5}$), the curve fully aligns with the
      \textcolor{DarkRed}{Logistic Loss (Red)}, reverting to the standard
      square-root rate ($y \propto \sqrt{x}$).  This confirms that the linear
      core is the necessary structural element for fast rates.}
    \label{fig:vanishing_core}
  \end{minipage}
\end{figure}

%%%%%%%%%%%%%%%%%%%%%%%%%%%%%%%%%%%%%%%%%%%%%%%%%%%%%%%%%%%%%%%%%%%%%%%%%%%%%%%
\section{Multi-Class Classification}
\label{sec:multi}
%%%%%%%%%%%%%%%%%%%%%%%%%%%%%%%%%%%%%%%%%%%%%%%%%%%%%%%%%%%%%%%%%%%%%%%%%%%%%%%

Let $\sY$ be a finite label set with cardinality $|\sY| = n \geq 2$. We consider
score functions $h \colon \sX \times \sY \to \Rset$, where the vector
$h(x, \cdot) \in \Rset^n$ represents the scores assigned to each class.  For any
instance $(x, y) \in \sX \times \sY$, we define the pairwise margins as
$ m_{y, y'}(x) \coloneqq h(x, y) - h(x, y')$ for all $ y' \in \sY$.  The target
$\ell_{\tar}$ is typically the multi-class zero-one loss $\sfL_{0-1}$, defined
as $ \sfL_{0-1}(h, x, y) = 1_{\hh(x) \neq y}$, where
$\hh(x) = \argmax_{y' \in \sY} h(x, y')$ denotes the label predicted by $h$ for
the input $x$.

\paragraph{Sum losses.} Generalizing the formulation of
\citet{weston1998multi}, we work with \emph{sum} surrogates of the form
\begin{equation}
  \label{eq:sum_loss}
  \ell^{\mathrm{sum}}_{\Phi}(h, x, y) = \sum_{y' \neq y} \Phi \paren*{-m_{y, y'}(x) }.
\end{equation}
(Equivalently, one may sum over $y' \in \sY$; this choice differs only by an
additive constant when $\Phi(0)$ is finite and does not affect minimization.) We
define the multi-class smooth surrogates by replacing $\Phi$
in~\eqref{eq:sum_loss} with either the symmetric linear-core surrogate
$\ov \Phi$ or the one-sided smoothing $\wt \Phi$:
\begin{equation*}
  \ell^{\mathrm{sum}}_{\ov \Phi}(h, x, y) = \sum_{y' \neq y} \ov \Phi \paren*{ m_{y, y'}(x) },
  \qquad
  \ell^{\mathrm{sum}}_{\wt \Phi}(h, x, y) = \sum_{y' \neq y} \wt \Phi \paren*{ m_{y, y'}(x) }.
\end{equation*}

\subsection{Convexity and Smoothness}
\label{sec:multi-smooth}

These multi-class surrogates preserve the desirable analytical properties of the
original binary surrogates.

\begin{proposition}[Convexity and Smoothness, multi-class]
  \label{prop:properties-multi}
  Let $\phi \in \curl*{\ov \Phi, \wt \Phi}$.  The multi-class Linear-Core
  surrogate $\ell^{\mathrm{sum}}_{\phi}$ is convex in the score vector
  $h(x, \cdot)$ and continuously differentiable ($C^1$). Furthermore, if the
  base $\Phi$ is twice continuously differentiable with $\Phi''(0)=0$, it is
  twice continuously differentiable ($C^2$).
\end{proposition}

\begin{proof}
  The multi-class loss
  $\ell^{\mathrm{sum}}_{\phi}(h, x, y) = \sum_{y' \neq y} \phi(h(x, y) - h(x,
  y'))$ is a non-negative linear combination of terms of the form $\phi(L(h))$,
  where $L$ is a linear functional of the score vector $h(x, \cdot)$.  Since
  $\phi$ is convex (Proposition~\ref{prop:properties}) and $L$ is linear,
  $\phi \circ L$ is convex. The sum is therefore convex.  Since $\phi$ is $C^1$
  (or $C^2$), and $L$ is smooth, the composition is $C^1$ (or $C^2$). Thus, the
  multi-class loss inherits the smoothness properties of the base scalar
  surrogate.
\end{proof}

Thus, the losses $\ell^{\mathrm{sum}}_{\ov \Phi}$ and
$\ell^{\mathrm{sum}}_{\wt \Phi}$ are convex in the score vector $h(x, \cdot)$
and are globally $C^1$, and even $C^2$ under mild conditions on $\Phi$.

\paragraph{Illustrative examples.}
To illustrate the behavior of the multi-class linear-core surrogates, we present
two concrete instances derived from the logistic and exponential losses.
Figure~\ref{fig:multi_surfaces} visualizes the resulting surfaces for a
three-class problem, plotting the loss
$\sum_{y' \neq y} \ov \Phi(h(x, y) - h(x, y'))$ as a function of the pairwise
margins $m_1 = h(x, y) - h(x, y_1)$ and $m_2 = h(x, y) - h(x, y_2)$; the linear
plateau in the central region reflects the linear-core structure.

\begin{figure}[t]
  \centering
  \includegraphics[width=0.4\linewidth]{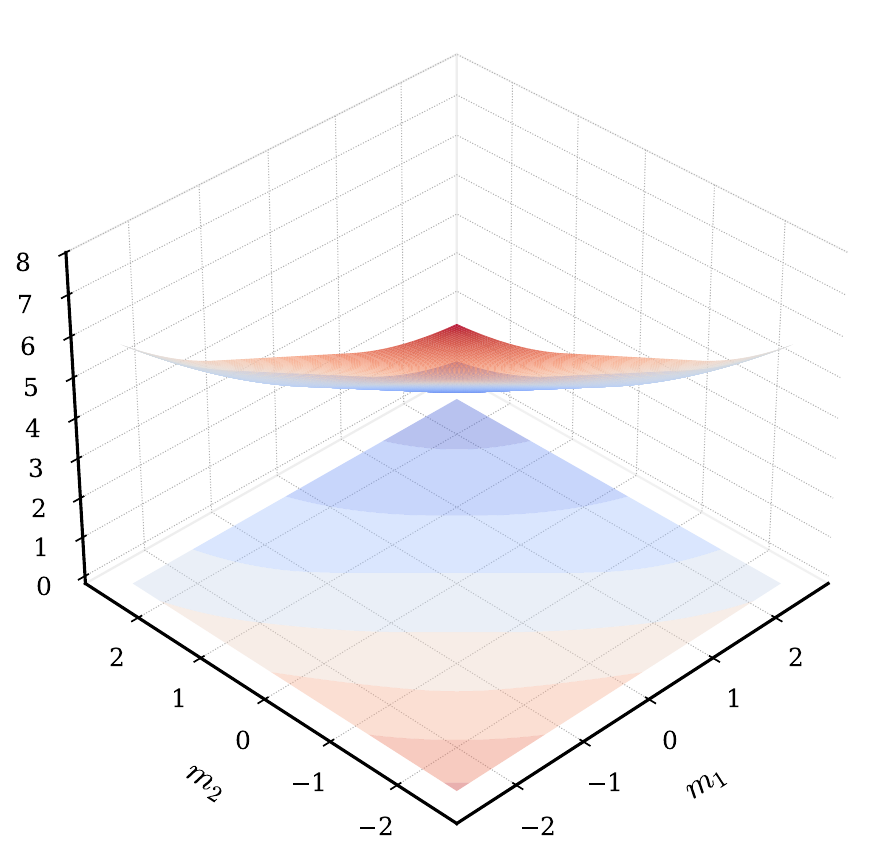}
  \hspace{0.06\linewidth}
  \includegraphics[width=0.4\linewidth]{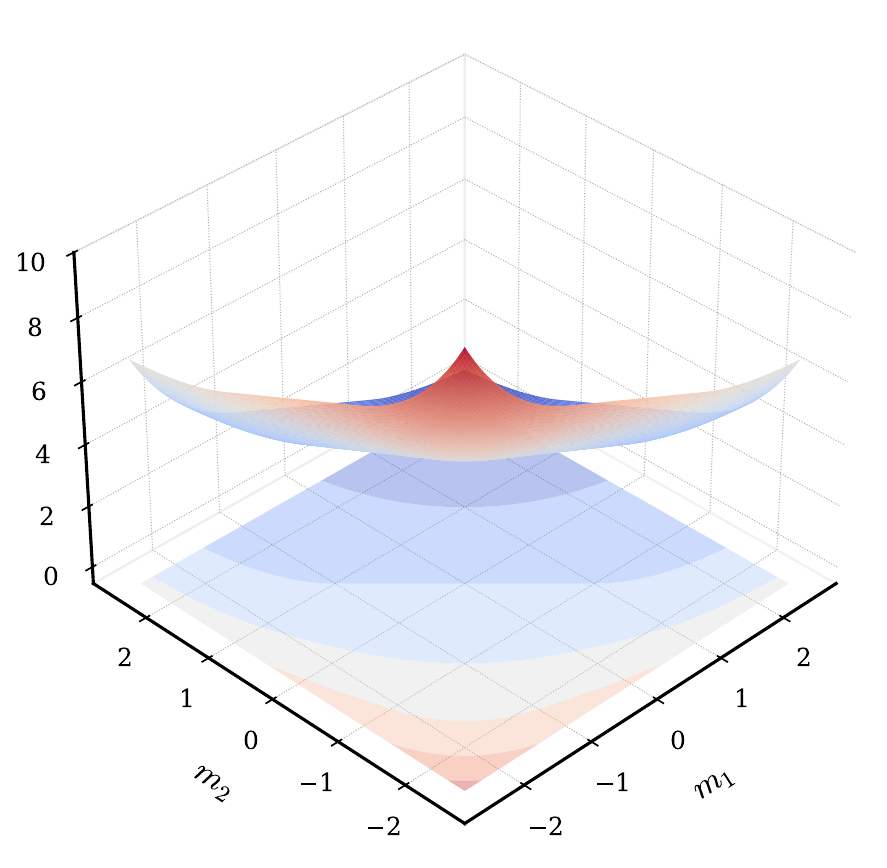}
  \caption{Multi-class sum loss surfaces $\ell^{\mathrm{sum}}_{\ov \Phi}$ for a
    3-class problem. Left: Logistic LC Surrogate. Right: Exponential LC
    Surrogate.}
  \label{fig:multi_surfaces}
\end{figure}

\emph{Logistic Linear-Core Surrogate.} Let $\Phi(u) = \log(1 + e^u)$. We have
$\Phi(0) = \log 2$ and $\Phi'(0) = 1/2$. The corresponding surrogate
$\ov \Phi_{\log}$ is given by:
\begin{equation*}
  \ov \Phi_{\log}(u) =
  \begin{cases}
    - u + 1 + 2 \log 2, & -1 \leq u \leq 1, \\[4pt]
    2 \log(1 + e^{1 - u}), & u > 1, \\[4pt]
    2 \log(1 + e^{-1 - u}) + 2, & u < -1.
  \end{cases}
\end{equation*}
In the three-class setting, the total loss is the sum of the surrogates applied
to each margin:
$\ell^{\mathrm{sum}}_{\ov \Phi_{\log}}(m_1, m_2) = \ov \Phi_{\log}(m_1) + \ov
\Phi_{\log}(m_2)$.  This function behaves linearly for small margins and
transitions smoothly to the scaled logistic tail for large positive margins.

\emph{Exponential Linear-Core Surrogate.} Let $\Phi(u) = e^u$. We have
$\Phi(0) = 1$ and $\Phi'(0) = 1$. The surrogate $\ov \Phi_{\exp}$ is:
\begin{equation*}
  \ov \Phi_{\exp}(u) =
  \begin{cases}
    - u + 2, & -1 \leq u \leq 1, \\[4pt]
    e^{1 - u}, & u > 1, \\[4pt]
    e^{-1 - u} + 2, & u < -1.
  \end{cases}
\end{equation*}
Similarly, the total loss for the three-class case is given by
$\ell^{\mathrm{sum}}_{\ov \Phi_{\exp}}(m_1, m_2) = \ov \Phi_{\exp}(m_1) + \ov
\Phi_{\exp}(m_2)$.  This creates a loss that is linear in the central region
$[-1, 1]$ and decays exponentially for large positive margins, offering a robust
alternative to the standard sum-exponential loss.

\subsection{Linear \texorpdfstring{$\sH$}{H}-Consistency Bound}
\label{sec:multi-bound}

We now establish a linear $\sH$-consistency bound with respect to the 0-1
loss. In contrast with the squared-hinge analysis (which leads to a
$\sqrt{\cdot}$ bound), the linear-core surrogates yield an exact linear bound
because their central branch is affine, guaranteeing fast-rate convergence to
the best-in-class classifier under minimal assumptions.  We begin with two
auxiliary lemmas.  The first gives the conditional regret of the 0-1 loss.

\begin{lemma}
  \label{lemma:explicit_assumption_01}
  Assume $\sH$ is symmetric and complete.  Then, for any $x \in \sX$, the
  best-in-class conditional error and the conditional regret for $\sfL_{0-1}$
  can be expressed as follows:
  \begin{align*}
    \sC^*_{\sfL_{0-1}}(\sH, x) & = 1 - \max_{y\in \sY} p(y \mid x)\\
    \Delta \sC_{\sfL_{0-1}, \sH}(h, x) & = \max_{y\in \sY} p(y \mid x) - p(\hh(x) \mid x).
  \end{align*}
\end{lemma}

\begin{proof}
  By \citet[Lemma~3]{awasthi2022multi} and the fact that $\mathsf{H}(x) = \sY$
  when $\sH$ is symmetric, the proof is complete.
\end{proof}

The second lemma solves in closed form the two-class, one-dimensional problem
obtained by restricting the margin to the linear core.

\begin{lemma}[Restricted optimizer for $\ov\Phi$ and $\wt \Phi$ on
  $\bracket*{-1, 1}$]
  \label{lem:restricted-interval-optimizer}
  For $a, b \geq 0$,
  \begin{equation*}
    \inf_{u \in \bracket*{-1, 1}} \paren*{ a \, \ov\Phi \paren*{ -u } + b \, \ov\Phi \paren*{ u } }
    = \inf_{u \in \bracket*{-1, 1}} \paren*{ a \, \wt\Phi \paren*{ -u } + b \, \wt \Phi \paren*{ u } }
    = \paren*{ a + b } \, \frac{\Phi(0)}{\Phi'(0)} + 2 \min \curl*{ a, b },
  \end{equation*}
  with the infimum attained at $u^* = -1$ if $a \geq b$ and at $u^* = 1$ if
  $a \leq b$.
\end{lemma}

\begin{proof}
  For $u \in \bracket*{-1, 1}$, the middle branch of $\ov\Phi$ and $\wt\Phi$
  gives
  \begin{equation*}
    \ov\Phi \paren*{ -u } = \wt\Phi \paren*{ -u } = 1 + u + \frac{\Phi(0)}{\Phi'(0)}, \qquad
    \ov\Phi \paren*{ u } = \wt\Phi \paren*{ u } = 1 - u + \frac{\Phi(0)}{\Phi'(0)}.
  \end{equation*}
  Therefore, for $a, b \geq 0$,
  \begin{equation*}
    a \, \ov\Phi \paren*{ -u } + b \, \ov\Phi \paren*{ u } = 
    a \, \wt\Phi \paren*{ -u } + b \, \wt\Phi \paren*{ u } 
    = \paren*{ a + b } \paren*{ 1 + \frac{\Phi(0)}{\Phi'(0)} } + \paren*{ a - b } u.
  \end{equation*}
  The right-hand side is affine in $u$, hence minimized on the interval
  $\bracket*{-1, 1}$ at an endpoint: at $u^* = -1$ if $a \geq b$, and at
  $u^* = 1$ if $a \leq b$.  Evaluating at these points yields
  \begin{align*}
    \inf_{u \in \bracket*{-1, 1}} \paren*{ a \, \ov\Phi \paren*{ -u } + b \, \ov\Phi \paren*{ u } } 
    & = \inf_{u \in \bracket*{-1, 1}} \paren*{ a \, \wt\Phi \paren*{ -u } + b \, \wt\Phi \paren*{ u } }\\
    & = \paren*{ a + b } \paren*{ 1 + \frac{\Phi(0)}{\Phi'(0)} } - \abs{ a - b }\\
    & = \paren*{ a + b } \, \frac{\Phi(0)}{\Phi'(0)} + 2 \min \curl*{ a, b }.
  \end{align*}
  This proves the claim.
\end{proof}

We are now ready to state and prove the main result of this section.

\begin{theorem}[Linear $\sH$-consistency bound for multi-class linear-core
  surrogates]
  \label{thm:bound-multi}
  Assume $\sH$ is symmetric and complete.  Then, for any distribution and any
  $h \in \sH$,
  \begin{align*}
    \sE_{\sfL_{0-1}}(h) - \sE^{*}_{\sfL_{0-1}}(\sH) + \sM_{\sfL_{0-1}}(\sH)
    & \leq \sE_{\ell^{\mathrm{sum}}_{\ov \Phi} }(h) 
      - \sE^{*}_{ \ell^{\mathrm{sum}}_{\ov \Phi} }(\sH)
      + \sM_{ \ell^{\mathrm{sum}}_{\ov \Phi}}\paren*{\sH }, \\
    \sE_{\sfL_{0-1}}(h) - \sE^{*}_{\sfL_{0-1}}(\sH) + \sM_{\sfL_{0-1}}(\sH)
    & \leq
      \sE_{\ell^{\mathrm{sum}}_{\wt \Phi} }(h)
      - \sE^{*}_{\ell^{\mathrm{sum}}_{\wt \Phi} }(\sH)
      + \sM_{\ell^{\mathrm{sum}}_{\wt \Phi}}\paren*{\sH}.
  \end{align*}
\end{theorem}

\paragraph{Proof idea.}
We decompose the conditional surrogate regret into a sum of pairwise regrets
between labels. By lower bounding the total regret using only the specific pair
$(y_{\max}, \hh(x))$, we reduce the problem to the binary case. Since the
Linear-Core loss has a constant gradient of magnitude $1$ at the origin (unlike
the vanishing gradient of smooth losses), the pairwise regret provides a linear
lower bound on the probability difference
$p(y_{\max}\mid x) - p(\hh(x) \mid x)$, which corresponds exactly to the 0-1
conditional regret (Lemma~\ref{lemma:explicit_assumption_01}).

\begin{proof}
  Fix $x \in \sX$. Let $\phi \in \{ \ov\Phi, \wt\Phi \}$.  For brevity, let
  $p_y \coloneqq p(y \mid x)$.  Let $y_{\max} \in \argmax_{y \in \sY} p_y$ and
  let $\hh(x) = \argmax_{y \in \sY} h(x, y)$.  If $\hh(x) = y_{\max}$, by
  Lemma~\ref{lemma:explicit_assumption_01}, the inequality
  $\Delta \sC_{\sfL_{0-1}, \sH}(h, x) \leq \Delta
  \sC_{\ell^{\mathrm{sum}}_{\phi}, \sH}(h, x)$ holds trivially since the
  left-hand side is zero.  Assume $\hh(x) \neq y_{\max}$.

  \textbf{Step 1: Pairwise decomposition.}
  The conditional error of the sum loss can be decomposed into a sum of
  pairwise errors. We have:
  \begin{align*}
    \sC_{\ell^{\mathrm{sum}}_{\phi}}(h, x)
    &= \sum_{y \in \sY} p_y \sum_{y' \neq y} \phi\paren*{ h(x, y) - h(x, y') } \\
    &= \frac12 \sum_{\{y, y'\} \subseteq \sY, y \neq y'} \underbrace{\bracket*{ p_y \, \phi\paren*{ h(x, y) - h(x, y') } + p_{y'} \, \phi\paren*{ h(x, y') - h(x, y) } }}_{\eqqcolon \sC_{\{y, y'\}}(h, x)}.
  \end{align*}
  We first determine the best-in-class conditional error
  $\sC^*_{\ell^{\mathrm{sum}}_{\phi}}(\sH, x)$.  Consider any pair $\{y, y'\}$
  with $p_y \geq p_{y'}$. Minimizing the pairwise term
  $\sC_{\{y, y'\}}(h, x)$ requires the margin $h(x, y) - h(x, y')$ to be
  optimized (typically driven to a positive value).  Generally, pairwise
  constraints might conflict (e.g., violating the triangle inequality).
  However, here the ``preference'' direction for every pair is determined by
  the order of the scalar probabilities $p(\cdot \mid x)$.  Since these
  probabilities induce a total ordering on $\sY$, the pairwise requirements are
  transitive and acyclic.  Therefore, there is no conflict: one can construct a
  score vector $h(x, \cdot)$ that satisfies the optimal margin requirements for
  \emph{all} pairs simultaneously (for instance, by setting scores
  proportional to the rank of $p_y$).  Since $\sH$ is complete, such a vector
  exists in $\sH$.  Thus, the infimum of the sum is the sum of the infimums:
  \begin{equation*}
    \sC^*_{\ell^{\mathrm{sum}}_{\phi}}(\sH, x) = \frac12 \sum_{\{y, y'\} \subseteq \sY, y \neq y'} \inf_{h \in \sH} \sC_{\{y, y'\}}(h, x).
  \end{equation*}
  The conditional regret then decomposes additively:
  \begin{equation*}
    \Delta \sC_{\ell^{\mathrm{sum}}_{\phi}, \sH}(h, x)
    = \frac12 \sum_{\{y, y'\} \subseteq \sY, y \neq y'} \underbrace{\sC_{\{y, y'\}}(h, x) - \inf_{h \in \sH} \sC_{\{y, y'\}}(h, x)}_{\eqqcolon \Delta \sC_{\{y, y'\}}(h, x)}.
  \end{equation*}

  \textbf{Step 2: Reduction to a single pair.}
  Since each pairwise regret term $\Delta \sC_{\{y, y'\}}(h, x)$ is
  non-negative, we can lower bound the total regret by the two terms
  corresponding to the pair $\{y_{\max}, \hh(x)\}$. Note that
  $\Delta \sC_{\{y, y'\}}(h, x) = \Delta \sC_{\{y', y\}}(h, x) $.  Let
  $y_1 = y_{\max}$ and $y_2 = \hh(x)$. By definition, $p_{y_1} \geq p_{y_2}$.
  Also, let $m = h(x, y_1) - h(x, y_2)$. Since $y_2$ is the predicted class,
  $h(x, y_2) \geq h(x, y_1)$, implying $m \leq 0$.  The pairwise regret for
  $\{y_1, y_2\}$ is:
  \begin{equation*}
    \Delta \sC_{\{y_1, y_2\}}(h, x)
    = p_{y_1} \phi(m) + p_{y_2} \phi(-m) - \inf_{u \in \Rset} \paren*{ p_{y_1} \phi(u) + p_{y_2} \phi(-u) }.
  \end{equation*}

  \textbf{Step 3: Linear lower bound on the pairwise regret.}
  We apply Lemma~\ref{lem:restricted-interval-optimizer}. Since
  $[-1, 1] \subset \Rset$, the infimum over $\Rset$ is upper bounded by the
  restricted infimum over $[-1, 1]$ (attained at the boundary $u = 1$ since
  $p_{y_1} \geq p_{y_2}$), so:
  \begin{equation*}
    \inf_{u \in \Rset} \paren*{ p_{y_1} \phi(u) + p_{y_2} \phi(-u) } \leq (p_{y_1} + p_{y_2})\frac{\Phi(0)}{\Phi'(0)} + 2 \min\{p_{y_1}, p_{y_2}\} = (p_{y_1} + p_{y_2})\frac{\Phi(0)}{\Phi'(0)} + 2 p_{y_2}.
  \end{equation*}
  For the first term, we use the property that $\phi$ has slope $-1$ at the
  origin. By convexity,
  $\phi(t) \geq \phi(0) - t = (1 + \frac{\Phi(0)}{\Phi'(0)}) - t$.  Thus:
  \begin{align*}
    p_{y_1} \phi(m) + p_{y_2} \phi(-m)
    &\geq p_{y_1} \bracket*{ 1 + \frac{\Phi(0)}{\Phi'(0)} - m } + p_{y_2} \bracket*{ 1 + \frac{\Phi(0)}{\Phi'(0)} + m } \\
    &= (p_{y_1} + p_{y_2}) \paren*{ 1 + \frac{\Phi(0)}{\Phi'(0)} } + (p_{y_2} - p_{y_1}) m.
  \end{align*}
  Subtracting the minimal error:
  \begin{align*}
    \Delta \sC_{\{y_1, y_2\}}(h, x)
    &\geq (p_{y_1} + p_{y_2}) \paren*{ 1 + \frac{\Phi(0)}{\Phi'(0)} } + (p_{y_2} - p_{y_1}) m - \bracket*{ (p_{y_1} + p_{y_2})\frac{\Phi(0)}{\Phi'(0)} + 2 p_{y_2} } \\
    &= (p_{y_1} + p_{y_2}) - 2 p_{y_2} + (p_{y_2} - p_{y_1}) m \\
    &= (p_{y_1} - p_{y_2}) + (p_{y_2} - p_{y_1}) m \\
    &= (p_{y_1} - p_{y_2}) (1 - m).
  \end{align*}

  \textbf{Step 4: Conclusion.}
  By the symmetry, we have
  $\Delta \sC_{\{y_2, y_1\}}(h, x) = \Delta \sC_{\{y_1, y_2\}}(h, x) \geq
  (p_{y_1} - p_{y_2}) (1 - m)$.  Since $p_{y_1} \geq p_{y_2}$ and $m \leq 0$, we
  have $(p_{y_1} - p_{y_2}) \geq 0$ and $(1 - m) \geq 1$. Therefore, since each
  pairwise regret is non-negative, we have:
  \begin{equation*}
    \Delta \sC_{\ell^{\mathrm{sum}}_{\phi}, \sH}(h, x) \geq \frac12 \paren*{\Delta \sC_{\{y_1, y_2\}}(h, x) +  \Delta \sC_{\{y_2, y_1\}}(h, x)} \geq p_{y_1} - p_{y_2} = p(y_{\max} \mid x) - p(\hh(x) \mid x).
  \end{equation*}
  By Lemma~\ref{lemma:explicit_assumption_01}, this lower bound equals
  $\Delta \sC_{\sfL_{0-1}, \sH}(h, x)$.  Finally, taking expectations over $x$
  yields the statement of the theorem:
  \begin{equation*}
    \sE_{\sfL_{0-1}}(h) - \sE^{*}_{\sfL_{0-1}}(\sH) + \sM_{\sfL_{0-1}}(\sH)
    \leq \sE_{\ell^{\mathrm{sum}}_{\phi} }(h) 
    - \sE^{*}_{ \ell^{\mathrm{sum}}_{\phi} }(\sH)
    + \sM_{ \ell^{\mathrm{sum}}_{\phi}}\paren*{\sH }. \qedhere
  \end{equation*}
\end{proof}

The \emph{sum} formulation in \eqref{eq:sum_loss} aggregates pairwise margins
against all competing labels, which has two pleasant consequences in our
setting.  First, replacing $\Phi$ by either linear-core surrogate $\ov\Phi$ or
$\wt\Phi$ preserves convexity in the score vector $h(x, \cdot)$ and grants
global $C^1$-smoothness, and even $C^2$-smoothness under mild assumptions on
$\Phi$ (Proposition~\ref{prop:properties-multi}).  Second, the pairwise
structure lets us reduce conditional regret lower bounds to a family of
\emph{two-class} one-dimensional optimization problems that can be solved in
closed form (Lemma~\ref{lem:restricted-interval-optimizer}).  Compared with sum
squared-hinge/exponential surrogates (which yield a $\sqrt{\cdot}$ transfer)
\citep{awasthi2022multi}, the linear-core surrogates admit a \emph{linear}
$\sH$-consistency bound because their middle branch is affine with slope $-1$ at
the origin; this ensures that the pointwise supporting-line lower bound holds.

\subsection{Empirical Validation: Robustness to Noise}
\label{sec:multi_experiments}

While the previous sections established the theoretical properties of our
Linear-Core surrogates, their practical use is best demonstrated by their
robustness to realistic data corruption. Standard losses like Cross-Entropy
(CE), also known as logistic
loss~\citep{Verhulst1838,Verhulst1845,Berkson1944,Berkson1951}, are particularly
sensitive to \emph{Instance-Dependent Noise (IDN)}
\citep{berthon2021confidence,cheng2020learning,du2015modelling}. Unlike the
uniform noise model (also known as symmetric label noise)
\citep{van2015learning,ghosh2017robust}, IDN concentrates corruption near the
decision boundary, where the probability of mislabeling correlates with feature
ambiguity (e.g., an image of a ``Dog'' resembling a ``Wolf''). This reflects a
far more realistic noise model encountered in real-world applications, as human
annotators rarely mislabel unambiguous examples far from the decision boundary
\citep{xia2020part}.

We hypothesize that the robustness of the One-Sided Linear-Core surrogate stems
from its gradient saturation near this critical boundary region.  As illustrated
in Figure~\ref{fig:gradient_invariance}, the Cross-Entropy loss (logistic loss)
has non-zero curvature ($\Phi'' > 0$) at the margin $u = 0$, and its gradient
magnitude varies continuously with the distance to the boundary.
This allows the optimizer to reduce the total loss by making fine-grained shifts
to the decision boundary to accommodate ambiguous, noisy examples.  In contrast,
the One-Sided Linear-Core surrogate is strictly affine for all margins
$u \le 1$.  Consequently, its gradient is locally constant (invariant) with
respect to the margin for misclassified and near-boundary examples (see
Figure~\ref{fig:gradient_invariance}, blue line). This prevents the optimizer
from shifting the decision boundary to minimize the individual losses of
corrupted points, effectively acting as a `hard' regularizer that counts margin
violations rather than fitting their probability estimates.

\begin{figure}[t]
  \centering
  \includegraphics[width=0.5\linewidth]{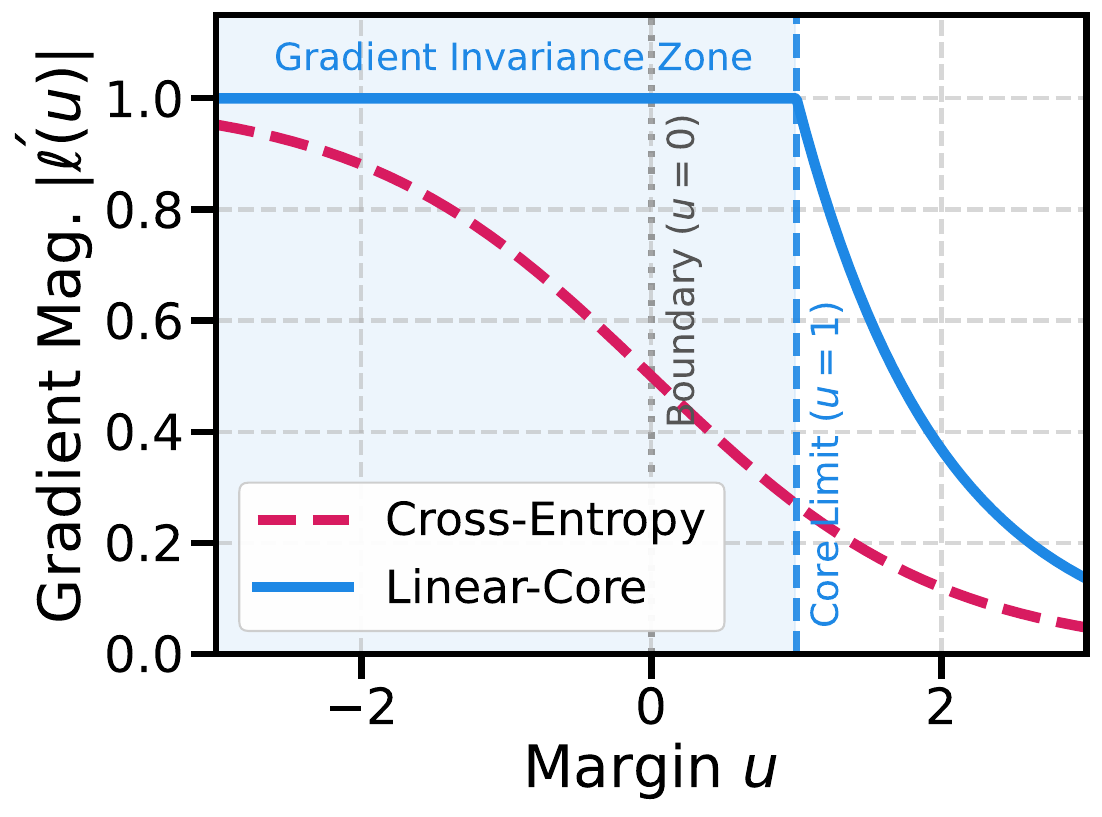}
  \caption{Gradient magnitudes: Linear-Core (Blue, saturated) vs. Logistic (Red,
    variable).}
  \label{fig:gradient_invariance}
\end{figure}

To validate this, we compare our surrogate against the Cross-Entropy baseline on
CIFAR-10 \citep{Krizhevsky09learningmultiple} under feature-dependent label
noise. We also compare our method against the \emph{Generalized Cross-Entropy
  (GCE)} loss \citep{zhang2018generalized}, which has been shown to be
empirically robust against uniform label noise.

\paragraph{Setup.}  We train a ResNet-18 \citep{he2016deep} using Stochastic
Gradient Descent (SGD) with momentum $0.9$ \citep{nesterov1983method}, weight
decay $5 \times 10^{-4}$, and a batch size of 128. The learning rate is
initialized at $0.1$ and annealed using a cosine schedule for 50 epochs.  We
introduce instance-dependent noise following the protocol of
\citet{xia2020part}: we project image features onto a random decision boundary
to generate flip probabilities, ensuring that visually ambiguous images are
significantly more likely to be mislabeled. We test noise rates
$\rho \in \{20\%, 30\%, 40\%, 50\%, 60\%\}$. We perform a grid search for the
hyperparameter $q \in (0, 1]$ of the GCE loss and report the best performance
for each noise rate to ensure a strong baseline.

\begin{table}[t]
  \centering
  \caption{Test Accuracy on CIFAR-10 under Instance-Dependent Noise. Linear-Core
    outperforms baselines.}
  \label{tab:noise_robustness}
  \vspace{2pt}
  \begin{tabular}{ccccc}
    \toprule
    IDN Rate ($\rho$) & Cross-Entropy & Gen.\ Cross-Entropy & Linear-Core & Improvement \\
                      & (CE) & (GCE, Best Tuned) & (Ours) & (vs.\ GCE) \\
    \midrule
    20\% & 83.20 & 83.23 & \textbf{84.24} & +1.01 \\
    30\% & 77.83 & 77.90 & \textbf{80.52} & +2.62 \\
    40\% & 72.88 & 72.93 & \textbf{75.49} & +2.56 \\
    50\% & 61.08 & 61.14 & \textbf{63.57} & +2.43 \\
    60\% & 37.50 & 37.57 & \textbf{39.86} & +2.29 \\
    \bottomrule
  \end{tabular}
\end{table}

\paragraph{Results.}  Table~\ref{tab:noise_robustness} summarizes the results. The
weakness of GCE under instance-dependent noise is evident: its performance
tracks the standard CE baseline almost identically across all noise rates (e.g.,
a negligible $0.05\%$ difference at $40\%$ noise), confirming that simply
re-weighting the loss is insufficient when noise mimics hard examples.  Our
Linear-Core surrogate, however, establishes a distinct performance gap. It
consistently outperforms GCE and CE across the entire spectrum of noise
rates. At a low noise rate ($20\%$), our method already demonstrates superior
generalization with a $+1.01\%$ gain. The advantage becomes most pronounced at
moderate noise levels ($30\%$--$40\%$), where our surrogate surpasses the tuned
GCE baseline by approximately $2.6\%$. Crucially, this robustness is sustained
even under severe corruption: at $50\%$ and $60\%$ noise, where the signal is
heavily degraded, our method maintains a consistent lead of approximately
$2.3\%$--$2.4\%$.  These results confirm that the constant gradient in the
linear core effectively suppresses the signal from systematically corrupted,
near-boundary examples where GCE and CE fail.

\subsubsection{Mechanism Analysis: Gradient Invariance}

To understand the source of this robustness, we analyzed the gradient dynamics
of the loss functions during training. At epoch 40 (after the learning rate
decay), we recorded the gradient magnitude $|\ell'(h(x))|$ for two distinct
groups of training examples: \emph{clean samples} (where the label is correct)
and \emph{noisy samples} on CIFAR-10 with 40\% instance-dependent noise.

Figure~\ref{fig:gradient_invariance_hist} visualizes the distribution of these
gradients, revealing a striking difference in behavior. As shown in
Figure~\ref{fig:gradient_invariance_hist} (Left), the Cross-Entropy loss assigns
a broad range of high-magnitude gradients ($0.6$ to $1.0$) to noisy
samples. This indicates that the loss function is actively ``negotiating'' with
outliers, assigning variable penalties based on the model's confidence, which
drives the decision boundary to overfit these corrupted points.
    
In contrast, Figure~\ref{fig:gradient_invariance_hist} (Right) shows that the
Linear-Core surrogate exhibits a sharp, Dirac-like peak exactly at
$|\ell'| \approx 1.0$ for noisy samples. This confirms our theoretical
hypothesis: for samples with negative margins ($u \le 1$), the gradient
saturates and becomes \emph{invariant} to the magnitude of the error. This
``hard'' clipping effectively ignores the degree of ``wrongness'' for outliers,
preventing the optimizer from shifting the boundary to accommodate mislabeled
examples.

\begin{figure}[t]
  \centering
  \includegraphics[width=0.44\linewidth]{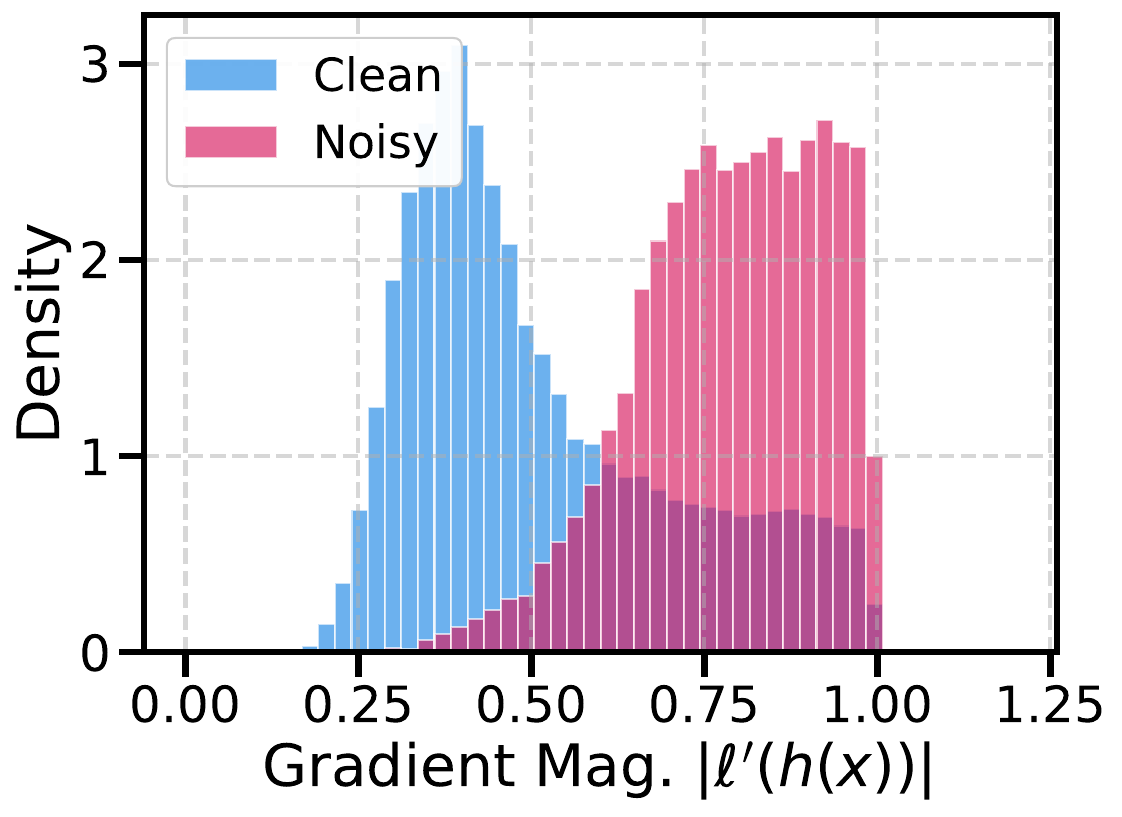}
  \hspace{0.04\linewidth}
  \includegraphics[width=0.44\linewidth]{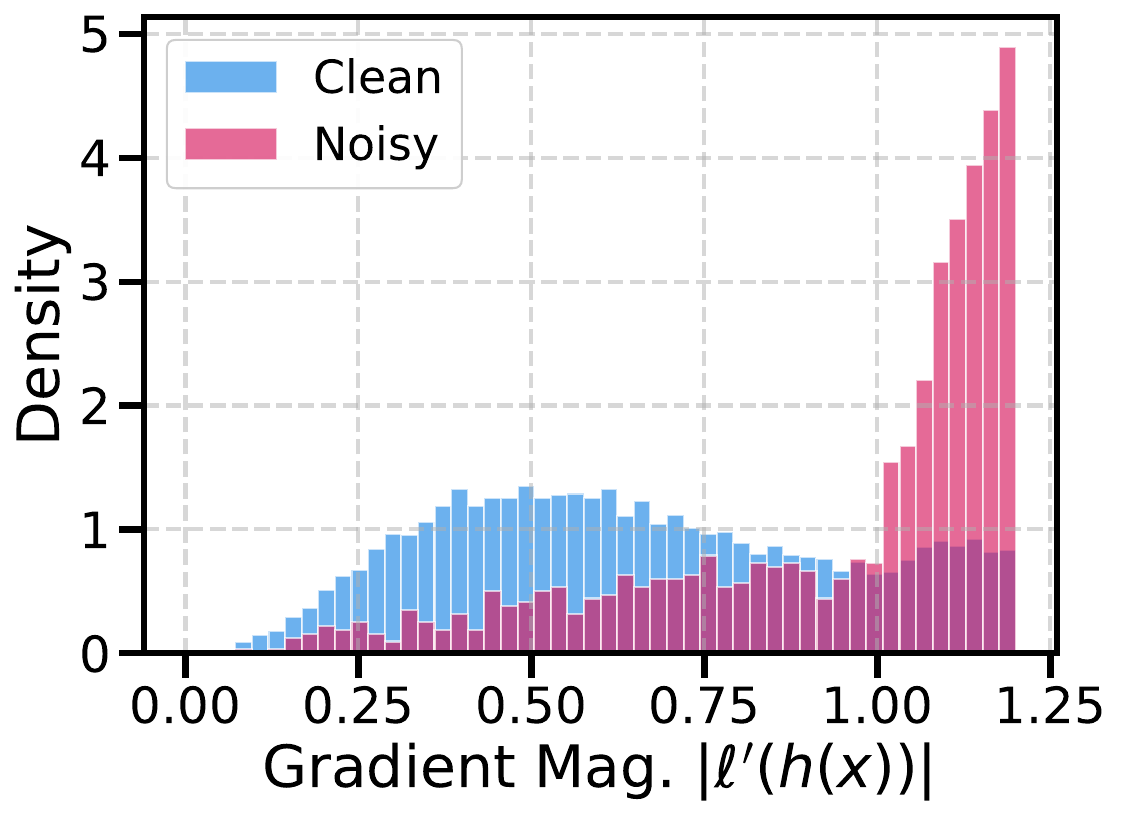}
  \caption{Gradient Magnitudes: CE (Left) vs. Linear-Core (Right).}
  \label{fig:gradient_invariance_hist}
\end{figure}

%%%%%%%%%%%%%%%%%%%%%%%%%%%%%%%%%%%%%%%%%%%%%%%%%%%%%%%%%%%%%%%%%%%%%%%%%%%%%%%
\section{Structured Prediction}
\label{sec:struct}
%%%%%%%%%%%%%%%%%%%%%%%%%%%%%%%%%%%%%%%%%%%%%%%%%%%%%%%%%%%%%%%%%%%%%%%%%%%%%%%

We consider the general structured prediction setting where the output space
$\sY$ may be exponentially large. We consider a target loss
$\loss(h, x, y) = \ell(\hh(x), y)$, where
$\ell \colon \sY \times \sY \to \Rset_+$ is a non-negative auxiliary loss
function (e.g., Hamming loss) such that $\ell(y, y) = 0$ for all $y \in \sY$.

Following \citet{MaoMohriZhong2023structured}, we work with \emph{structured sum
  losses}. Let $\ov \ell(y', y) = 1 - \ell(y', y)$ denote the similarity
score. The structured sum loss is defined as follows: for all
$(x, y) \in \sX \times \sY$,
\begin{equation*}
  \loss^{\mathrm{sum}}_{\Phi}(h, x, y) =
  \sum_{y' \in \sY} \ov \ell(y', y)
  \sum_{y''\neq y'} \Phi \paren*{ h(x, y'') - h(x, y') }.
\end{equation*}
This formulation effectively aggregates pairwise margins, weighted by the
structural similarity between the candidate $y'$ and the true label $y$.  We
define the structured prediction smooth surrogates by replacing the base
function $\Phi$ with either the symmetric linear-core surrogate $\ov \Phi$ or
the one-sided smoothing $\wt \Phi$:
\begin{align*}
  \loss^{\mathrm{sum}}_{\ov \Phi}(h, x, y) & = \sum_{y' \in \sY} \ov \ell(y', y)
                                             \sum_{y''\neq y'} \ov \Phi \paren*{m_{y', y''}(x)}, \\
  \loss^{\mathrm{sum}}_{\wt \Phi}(h, x, y) & = \sum_{y' \in \sY} \ov \ell(y', y)
                                             \sum_{y''\neq y'} \wt \Phi \paren*{m_{y', y''}(x)}.
\end{align*}

\subsection{Convexity and Smoothness}
\label{sec:struct-smooth}

The optimization landscape of the structured surrogate is determined by the
properties of $\ov \Phi$ and $\wt \Phi$.  The structural aggregation preserves
the convexity of the base scalar function, ensuring that the learning objective
remains amenable to global minimization, as well as its smoothness.

\begin{proposition}[Convexity and Smoothness, structured]
  \label{prop:properties-struct}
  Let $\phi \in \curl*{\ov \Phi, \wt \Phi}$.  The structured Linear-Core
  surrogate $\loss^{\mathrm{sum}}_{\phi}$ is convex in the score vector
  $h(x, \cdot)$ and continuously differentiable ($C^1$). Furthermore, if the
  base $\Phi$ is twice continuously differentiable with $\Phi''(0)=0$, it is
  twice continuously differentiable ($C^2$).
\end{proposition}

\begin{proof}
  As with the multi-class loss (Proposition~\ref{prop:properties-multi}), the
  structured loss $\loss^{\mathrm{sum}}_{\phi}$ is a non-negative linear
  combination of terms of the form $\phi(L(h))$, where $L$ is a linear
  functional of the score vector $h(x, \cdot)$.  Since $\phi$ is convex and $L$
  is linear, $\phi \circ L$ is convex, and the sum is therefore convex.  Since
  $\phi$ is $C^1$ (or $C^2$), and $L$ is smooth, the composition is $C^1$ (or
  $C^2$). Thus, the structured loss inherits the smoothness properties of the
  base scalar surrogate.
\end{proof}

\subsection{Linear \texorpdfstring{$\sH$}{H}-Consistency Bound}
\label{sec:struct-bound}

We now state the main consistency result for structured prediction. Similar to
the multi-class setting, the affine behavior of the linear-core surrogates near
the origin allows us to derive a linear bound relating the estimation error of
the surrogate to that of the target structured loss.  We first recall the form
of the conditional regret of the target structured loss.

\begin{lemma}
  \label{lemma:explicit_assumption_01-struct}
  Assume $\sH$ is symmetric and complete.  Then, for any $x \in \sX$, the
  best-in-class conditional error and the conditional regret for $\loss$ can be
  expressed as follows:
  \begin{align*}
    \sC^*_{\loss}(\sH, x) & = \min_{y' \in \sY} \sum_{y \in \sY} p(y \mid x) \ell(y', y)\\
    \Delta \sC_{\loss, \sH}(h, x) & = \sum_{y \in \sY} p(y \mid x) \ell(\hh(x), y) - \min_{y' \in \sY} \sum_{y \in \sY} p(y \mid x) \ell(y', y).
  \end{align*}
\end{lemma}

\begin{proof}
  By \citet[Lemma~3]{MaoMohriZhong2023structured} and the fact that
  $\mathsf{H}(x) = \sY$ when $\sH$ is symmetric, the proof is complete.
\end{proof}

\begin{theorem}[Linear $\sH$-consistency bound for structured prediction]
  \label{thm:bound-struct}
  Assume $\sH$ is symmetric and complete. Let $\loss$ be the target structured
  loss defined by $\ell$.  Then, for any distribution and any $h \in \sH$, the
  following bounds hold:
  \begin{align*}
    \sE_{\loss}(h) - \sE^{*}_{\loss}(\sH) + \sM_{\loss}(\sH)
    & \leq \sE_{\loss^{\mathrm{sum}}_{\ov \Phi} }(h)
      - \sE^{*}_{ \loss^{\mathrm{sum}}_{\ov \Phi} }(\sH)
      + \sM_{ \loss^{\mathrm{sum}}_{\ov \Phi}}\paren*{\sH }, \\
    \sE_{\loss}(h) - \sE^{*}_{\loss}(\sH) + \sM_{\loss}(\sH)
    & \leq
      \sE_{\loss^{\mathrm{sum}}_{\wt \Phi} }(h)
      - \sE^{*}_{\loss^{\mathrm{sum}}_{\wt \Phi} }(\sH)
      + \sM_{\loss^{\mathrm{sum}}_{\wt \Phi}}\paren*{\sH}.
  \end{align*}
\end{theorem}

\paragraph{Proof idea.}
The structured loss is defined as a sum of pairwise margins weighted by the
structural similarity $\ov{\ell}(y', y)$. We define a conditional regret
that aggregates these pairwise terms. Similar to the multi-class case, we
lower bound this sum by the contribution of the ``most violated'' pair
relative to the prediction. Because the local margin loss is linear near zero,
this contribution scales linearly with the target structural error, avoiding
the square-root degradation typical of sums of smooth convex functions.

\begin{proof}
  The proof for $\ov \Phi$ and $\wt \Phi$ is identical due to the coincidence of
  the functions on $[-1, 1]$.  Let $\phi \in \{ \ov\Phi, \wt\Phi \}$.  We first
  establish a pointwise lower bound on the surrogate regret.  Fix $x$. Let
  $p_y = p(y \mid x)$.  The target conditional error is
  $\sC_{\loss}(h, x) = \sum_{y \in \sY} p_y \ell(\hh(x), y)$. Since $\sH$ is
  complete, by Lemma~\ref{lemma:explicit_assumption_01-struct}, the
  best-in-class target conditional error is
  $\sC^*_{\loss}(\sH, x) = \inf_{y' \in \sY} \sum_{y \in \sY} p_y \ell(y', y)$.
  Thus, the target conditional regret is:
  \begin{equation*}
    \Delta \sC_{\loss, \sH}(h, x) = \sum_{y \in \sY} p_y \ell(\hh(x), y) - \inf_{y' \in \sY} \sum_{y \in \sY} p_y \ell(y', y).
  \end{equation*}

  \textbf{Step 1: Pairwise decomposition.}
  Now consider the surrogate loss.  The conditional surrogate error is:
  \begin{align*}
    \sC_{\loss^{\mathrm{sum}}_{\phi}}(h, x)
    &= \sum_{y \in \sY} p_y \sum_{y' \in \sY} \ov \ell(y', y) \sum_{y'' \neq y'} \phi(h(x, y') - h(x, y'')) \\
    &= \sum_{y' \in \sY} \sum_{y'' \neq y'} \phi(h(x, y') - h(x, y'')) \underbrace{\sum_{y \in \sY} p_y \ov \ell(y', y)}_{\eqqcolon W(y')}\\
    &= \frac12 \sum_{\{y', y''\} \subseteq \sY, y' \neq y''} \underbrace{\bracket*{ W(y') \phi\paren*{ h(x, y') - h(x, y'') } + W(y'') \phi\paren*{ h(x, y'') - h(x, y') } }}_{\eqqcolon \sC_{\{y', y''\}}(h, x)}.
  \end{align*}
  Note that the inner summation $W(y')$ does not depend on $y''$.  We first
  determine the best-in-class conditional error
  $\sC^*_{\loss^{\mathrm{sum}}_{\phi}}(\sH, x)$.  Consider any pair
  $\{y', y''\}$ with $W(y') \geq W(y'')$. Minimizing the pairwise term
  $\sC_{\{y', y''\}}(h, x)$ requires the margin $h(x, y') - h(x, y'')$ to be
  optimized (typically driven to a positive value).  Generally, pairwise
  constraints might conflict (e.g., violating the triangle inequality).
  However, here the ``preference'' direction for every pair is determined by the
  order of the scalar weights $W(\cdot)$.  Since these weights induce a total
  ordering on $\sY$, the pairwise requirements are transitive and acyclic.
  Therefore, there is no conflict: one can construct a score vector
  $h(x, \cdot)$ that satisfies the optimal margin requirements for \emph{all}
  pairs simultaneously (for instance, by setting scores proportional to the rank
  of $W(y)$).  Since $\sH$ is complete, such a vector exists in $\sH$.  Thus,
  the infimum of the sum is the sum of the infimums:
  \begin{equation*}
    \sC^*_{\loss^{\mathrm{sum}}_{\phi}}(\sH, x) = \frac12 \sum_{\{y', y''\} \subseteq \sY, y' \neq y''} \inf_{h \in \sH} \sC_{\{y', y''\}}(h, x).
  \end{equation*}
  The conditional regret then decomposes additively:
  \begin{equation*}
    \Delta \sC_{\loss^{\mathrm{sum}}_{\phi}, \sH}(h, x)
    = \frac12 \sum_{\{y', y''\} \subseteq \sY, y' \neq y''} \underbrace{\sC_{\{y', y''\}}(h, x) - \inf_{h \in \sH} \sC_{\{y', y''\}}(h, x)}_{\eqqcolon \Delta \sC_{\{y', y''\}}(h, x)}.
  \end{equation*}

  \textbf{Step 2: Reduction to a single pair.}
  Let $y_{\max} \in \argmax_{y \in \sY} W(y)$ and let
  $\hh(x) = \argmax_{y \in \sY} h(x, y)$.  If $\hh(x) = y_{\max}$, by
  Lemma~\ref{lemma:explicit_assumption_01-struct}, the inequality
  $\Delta \sC_{\loss, \sH}(h, x) \leq \Delta \sC_{\loss^{\mathrm{sum}}_{\phi},
    \sH}(h, x)$ holds trivially since the left-hand side is zero.  Assume
  $\hh(x) \neq y_{\max}$.  Since each pairwise regret term
  $\Delta \sC_{\{y', y''\}}(h, x)$ is non-negative, we can lower bound the total
  regret by the two terms corresponding to the pair $\{y_{\max}, \hh(x)\}$. Note
  that $\Delta \sC_{\{y', y''\}}(h, x) = \Delta \sC_{\{y'', y'\}}(h, x) $.  Let
  $y_1 = y_{\max}$ and $y_2 = \hh(x)$. By definition, $W(y_1) \geq W(y_2)$.
  Also, let $m = h(x, y_1) - h(x, y_2)$. Since $y_2$ is the predicted class,
  $h(x, y_2) \geq h(x, y_1)$, implying $m \leq 0$.  The pairwise regret for
  $\{y_1, y_2\}$ is:
  \begin{equation*}
    \Delta \sC_{\{y_1, y_2\}}(h, x)
    = W(y_1) \phi(m) + W(y_2) \phi(-m) - \inf_{u \in \Rset} \paren*{ W(y_1) \phi(u) + W(y_2) \phi(-u) }.
  \end{equation*}

  \textbf{Step 3: Linear lower bound on the pairwise regret.}
  We apply Lemma~\ref{lem:restricted-interval-optimizer}. Since
  $[-1, 1] \subset \Rset$, the infimum over $\Rset$ is upper bounded by the
  restricted infimum over $[-1, 1]$ (attained at the boundary $u = 1$ since
  $W(y_1) \geq W(y_2)$), so:
  \begin{align*}
    \inf_{u \in \Rset} \paren*{ W(y_1) \phi(u) + W(y_2) \phi(-u) } 
    & \leq (W(y_1) + W(y_2))\frac{\Phi(0)}{\Phi'(0)} + 2 \min\{W(y_1), W(y_2)\}\\
    & = (W(y_1) + W(y_2))\frac{\Phi(0)}{\Phi'(0)} + 2 W(y_2).
  \end{align*}
  For the first term, we use the property that $\phi$ has slope $-1$ at the
  origin. By convexity,
  $\phi(t) \geq \phi(0) - t = (1 + \frac{\Phi(0)}{\Phi'(0)}) - t$.  Thus:
  \begin{align*}
    W(y_1) \phi(m) + W(y_2) \phi(-m)
    &\geq W(y_1) \bracket*{ 1 + \frac{\Phi(0)}{\Phi'(0)} - m } + W(y_2) \bracket*{ 1 + \frac{\Phi(0)}{\Phi'(0)} + m } \\
    &= (W(y_1) + W(y_2)) \paren*{ 1 + \frac{\Phi(0)}{\Phi'(0)} } + (W(y_2) - W(y_1)) m.
  \end{align*}
  Subtracting the minimal error:
  \begin{align*}
    \Delta \sC_{\{y_1, y_2\}}(h, x)
    &\geq (W(y_1) + W(y_2)) \paren*{ 1 + \frac{\Phi(0)}{\Phi'(0)} } + (W(y_2) - W(y_1)) m\\
    & \qquad - \bracket*{ (W(y_1) + W(y_2))\frac{\Phi(0)}{\Phi'(0)} + 2 W(y_2) } \\
    &= (W(y_1) + W(y_2)) - 2 W(y_2) + (W(y_2) - W(y_1)) m \\
    &= (W(y_1) - W(y_2)) + (W(y_2) - W(y_1)) m \\
    &= (W(y_1) - W(y_2)) (1 - m).
  \end{align*}

  \textbf{Step 4: Conclusion.}
  By the symmetry, we have
  $\Delta \sC_{\{y_2, y_1\}}(h, x) = \Delta \sC_{\{y_1, y_2\}}(h, x) \geq
  (W(y_1) - W(y_2)) (1 - m)$.  Since $W(y_1) \geq W(y_2)$ and $m \leq 0$, we
  have $(W(y_1) - W(y_2)) \geq 0$ and $(1 - m) \geq 1$. Therefore, since each
  pairwise regret is non-negative, we have:
  \begin{equation*}
    \Delta \sC_{\loss^{\mathrm{sum}}_{\phi}, \sH}(h, x) \geq \frac12 \paren*{\Delta \sC_{\{y_1, y_2\}}(h, x) +  \Delta \sC_{\{y_2, y_1\}}(h, x)} \geq W(y_1) - W(y_2).
  \end{equation*}
  By Lemma~\ref{lemma:explicit_assumption_01-struct}, this lower bound equals
  $\Delta \sC_{\loss, \sH}(h, x)$.  Finally, taking expectations over $x$ yields
  the statement of the theorem:
  \begin{equation*}
    \sE_{\loss}(h) - \sE^{*}_{\loss}(\sH) + \sM_{\loss}(\sH)
    \leq \sE_{\loss^{\mathrm{sum}}_{\phi} }(h)
    - \sE^{*}_{ \loss^{\mathrm{sum}}_{\phi} }(\sH)
    + \sM_{ \loss^{\mathrm{sum}}_{\phi}}\paren*{\sH }. \qedhere
  \end{equation*}
\end{proof}

In contrast to structured sum-exponential surrogates, which yield a square-root
rate \citep{MaoMohriZhong2023structured}, this result establishes a
\emph{linear} rate. The structured linear-core surrogates thus serve as valid
smooth proxies for the discrete structured error with strictly stronger
$\sH$-consistency guarantees.

\begin{remark}[Comparison of Convergence Rates]
  To explicitly compare our theorems to previous works, we highlight the
  theoretical difference in the convergence transfer rate dictated by the
  geometry of the surrogate losses:
  \begin{itemize}
  \item \textbf{Binary Classification:} Previous works analyzing standard smooth
    surrogates (e.g., Logistic, Exponential, Squared-Hinge) establish a
    square-root $\sH$-consistency bound, meaning the excess target risk
    $\Delta\sR$ is bounded by $O(\sqrt{\Delta\sL})$
    \citep{bartlett2006convexity,Awasthi2022Hconsistency}. Piecewise-linear
    losses like the Hinge loss achieve a fast linear bound
    ($\Delta\sR \le O(\Delta\sL)$) but are non-differentiable
    \citep{Awasthi2022Hconsistency}. Our Theorem~\ref{thm:bound-binary} proves
    that the Linear-Core surrogate achieves the strict linear rate of the Hinge
    loss while remaining globally $C^1/C^2$ smooth.
  \item \textbf{Multi-Class Classification:} For multi-class sum losses,
    previous smooth variants (e.g., sum-exponential or sum-squared-hinge)
    similarly suffer from a square-root transfer bound
    \citep{awasthi2022multi}. Our Theorem~\ref{thm:bound-multi} shows that by
    aggregating pairwise margins using our novel linear-core base, the
    multi-class Linear-Core surrogate strictly improves this to a linear bound.
  \item \textbf{Structured Prediction:} Previous work on structured
    sum-exponential surrogates \citep{MaoMohriZhong2023structured} yields a
    square-root rate. Our Theorem~\ref{thm:bound-struct} demonstrates that the
    structured Linear-Core surrogate elevates this to a fast linear rate.
  \end{itemize}
  To our knowledge, Linear-Core Surrogates are the first explicit family of loss
  functions to simultaneously guarantee $C^1/C^2$ smoothness and strict linear
  $\sH$-consistency across binary, multi-class, and structured domains.
\end{remark}

\subsection{Optimization and Computational Efficiency}
\label{sec:struct_optimization}

\paragraph{Optimization guarantees.}
While the standard Hinge loss (and similarly other piecewise linear loss
functions) is non-differentiable only at isolated points, this lack of
smoothness fundamentally alters the available convergence guarantees. Non-smooth
convex optimization relies on sub-gradient methods, which are theoretically
limited to a slow convergence rate of $O(1/\sqrt{T})$. In contrast, by
establishing that our linear-core surrogates are globally $C^1$ and admit valid
second-order approximations (Propositions~\ref{prop:properties},
\ref{prop:properties-multi}, and~\ref{prop:properties-struct}), we enable the
use of smooth gradient-based optimizers. For smooth convex functions, standard
gradient descent guarantees a faster rate of $O(1/T)$, and Nesterov's
accelerated gradient methods can achieve the optimal rate of $O(1/T^2)$.

Furthermore, in the structured prediction setting, the loss
$\loss^{\mathrm{sum}}_{\ov \Phi}$ aggregates margins over an exponentially large
output space. Here, the non-differentiable \emph{kinks} of a standard Hinge loss
form a complex arrangement of hyperplanes rather than a single point, often
causing sub-gradient methods to oscillate and stall.  To verify this, we
conducted a controlled experiment on a synthetic isotropic binary classification
problem with orthogonal features.

We minimized both the Hinge loss and the Linear-Core surrogate using Stochastic
Gradient Descent (SGD) \citep{robbins1951stochastic} with a fixed learning rate
$\eta=0.1$ and $L_2$ regularization ($\lambda=0.05$). As empirically
demonstrated in Figure~\ref{fig:optimization_trajectories}, the guaranteed $C^1$
smoothness of our surrogates ensures that the gradient magnitude naturally
decays near the optimum. This eliminates the chattering phenomenon inherent to
the Hinge loss, where non-zero sub-gradients prevent settling, and allows the
optimizer to converge linearly to high precision.

\begin{figure}[t]
  \centering
  \includegraphics[width=0.5\linewidth]{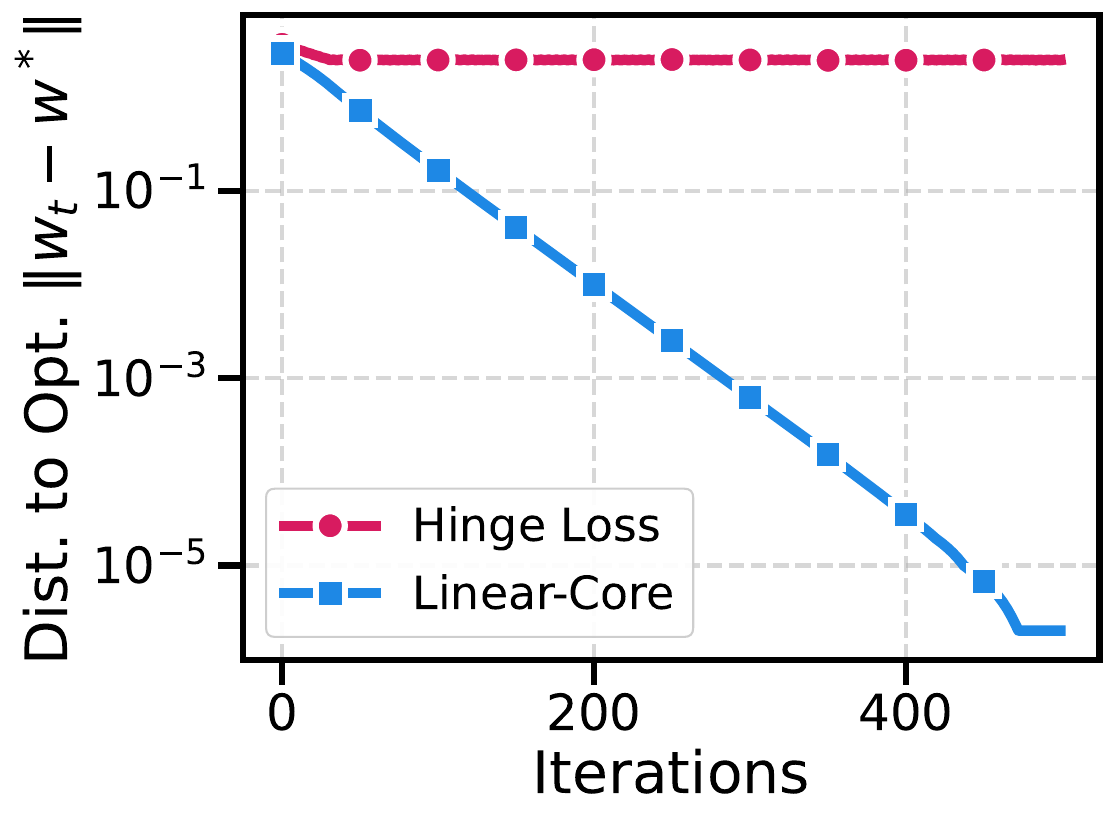}
  \caption{Optimization Stability: Linear-Core (Blue) converges linearly; Hinge
    (Red) stagnates.}
  \label{fig:optimization_trajectories}
\end{figure}

\paragraph{Computational efficiency via stochastic sampling.}
In the work of \citet{MaoMohriZhong2023structured}, the computational
tractability of the structured sum-exponential loss relies on the algebraic
homomorphism of the exponential function, $\exp(u+v) = \exp(u)\exp(v)$. This
property allows the sum over exponentially many structures to be decomposed into
local factors, enabling exact gradient computation via dynamic programming
algorithms (e.g., Forward-Backward or Sum-Product) in polynomial time. Our
proposed Structured Linear-Core Surrogates do not satisfy this multiplicative
property due to their piecewise definition and linear core. Consequently, exact
computation of the full sum loss over $\sY \times \sY$ is generally intractable
for large structured spaces.

However, we can ensure computational efficiency by exploiting the additive
structure of the loss. We interpret the structured sum loss as an expectation
over pairs of labels $(y', y'')$. Specifically, we can rewrite the gradient
update as an expectation under a sampling distribution $\sD$:
\begin{equation*}
  \nabla \loss^{\mathrm{sum}}_{\ov \Phi}(h, x, y)
  = \E_{y' \sim \sD_1, y'' \sim \sD_2}
  \bracket*{ \frac{\ov \ell(y', y)}{\P_{\sD}(y', y'')} \,
    \ov \Phi'(h(x, y'') - h(x, y')) \, \nabla (h(x, y'') - h(x, y')) }.
\end{equation*}
By constructing an unbiased estimator of the gradient using Monte Carlo sampling
of pairs $(y', y'')$, we reduce the per-iteration computational complexity from
$O(|\sY|^2)$ to $O(L)$ (the cost of sampling and embedding a single structure of
length $L$). For the outer summation, since $\ov \ell(y', y)$ typically
decomposes over the structure (e.g., Hamming distance), we can efficiently
sample $y'$ from a proposal distribution proportional to the structural error
(or simply uniformly with importance weights). For the inner summation, we
sample $y''$ via a simple proposal distribution (e.g., uniform or local
perturbation). The resulting procedure is summarized in
Algorithm~\ref{alg:struct_sgd}.

\begin{algorithm}[t]
  \caption{Stochastic Gradient Descent for Structured Linear-Core Surrogate}
  \label{alg:struct_sgd}
  \begin{algorithmic}[1]
    \STATE {\bfseries Input:} Training set $S$, step size $\eta$, proposal
    distributions $\sD_1, \sD_2$.
    \STATE {\bfseries Initialize:} $h_0 = 0$.
    \FOR{$t = 1$ {\bfseries to} $T$}
    \STATE Sample $(x, y)$ uniformly from $S$.
    \STATE \COMMENT{1. Sample outer label $y'$ based on structural error}
    \STATE Sample $y' \sim \sD_1(\cdot \mid y)$ (e.g., prop.\ to Hamming distance).
    \STATE Compute weight $w_1 = \frac{\ov{\ell}(y', y)}{\sD_1(y')}$.
    \STATE \COMMENT{2. Sample inner label $y''$ for comparison}
    \STATE Sample $y'' \sim \sD_2(\cdot \mid y')$ (e.g., uniform neighbor).
    \STATE Compute weight $w_2 = \frac{1}{\sD_2(y'')}$.
    \STATE \COMMENT{3. Compute Gradient Estimator}
    \STATE $m \gets h_{t-1}(x, y'') - h_{t-1}(x, y')$
    \STATE $g_t \gets w_1 \, w_2 \, \ov{\Phi}'(m) \cdot (\nabla h(x, y'') - \nabla h(x, y'))$
    \STATE {\bfseries Update:} $h_t \gets h_{t-1} - \eta g_t$
    \ENDFOR
    \STATE {\bfseries Output:} $h_T$
  \end{algorithmic}
\end{algorithm}

Crucially, unlike non-smooth structured losses which require solving a global
inference problem (Loss-Augmented Inference) at every step, our approach
requires only forward sampling. This makes each iteration extremely fast and
trivial to parallelize. Furthermore, the smoothness of $\ov \Phi$ ensures that
the variance of the gradient estimates remains bounded, preserving the
convergence guarantees of Stochastic Gradient Descent (SGD).

\paragraph{Variance of the stochastic gradient.}
A potential drawback of replacing exact inference with stochastic sampling is
the introduction of gradient noise. If the variance of the stochastic gradient
were to scale with the size of the output space $|\sY|$, the convergence rate
would degrade for large-vocabulary tasks.  We show that, remarkably, the
variance of our estimator depends only on the number of samples $K$ and the
feature radius $R$, and is \emph{independent of $|\sY|$}.

\begin{theorem}[Variance Bound for Stochastic Gradients]
  \label{thm:variance_bound}
  Let $\ell(\bw)$ be the one-sided Linear-Core surrogate loss (based on
  $\wt \Phi$, with $\Phi$ non-decreasing). Let $\h \nabla \ell(\bw)$
  be the stochastic gradient estimator constructed using a mini-batch of $K$
  negative samples $\{\by_k\}_{k=1}^K$ drawn uniformly from $\sY$. Assume the
  feature map $\Psi$ is bounded such that $\|\Psi(\bx, \by)\|_2 \le R$ for all
  $\bx, \by$. Then, the variance of the estimator is bounded by:
  \begin{equation*}
    \E\bracket*{ \norm*{ \h \nabla \ell(\bw) - \nabla \ell(\bw) }_2^2 } \le
    \frac{4 R^2}{K}.
  \end{equation*}
\end{theorem}

\paragraph{Proof idea.}
The gradient estimator is an average of $K$ independent terms $\bg_k$. Since
the surrogate $\wt \Phi$ is $1$-Lipschitz (the derivative is bounded by 1), the
norm of any single gradient term is bounded by
$\|\bg_k\| \le 1 \cdot \diam(\Psi) \le 2R$. By properties of variance for
independent bounded variables,
$\Var(\h \nabla) \le \frac{1}{K} \sup \| \bg \|^2 \le \frac{4R^2}{K}$.
Crucially, this bound relies only on the feature geometry $R$, not the
cardinality $|\sY|$.

\begin{proof}
  Recall that the gradient of the Linear-Core loss for a single example
  $(\bx, \by^*)$ can be written as an expectation:
  \begin{equation*}
    \nabla \ell(\bw)
    = \E_{\by \sim \P(\by|\bx)} \bracket*{ \phi'(\cdot)
      \paren*{ \Psi(\bx, \by) - \Psi(\bx, \by^*) } },
  \end{equation*}
  where $\phi'$ is the scalar derivative of the surrogate and $\P(\by|\bx)$ is
  the sampling distribution (e.g., uniform).  Let $\bg_k$ be the gradient
  estimate from a single sample $\by_k$:
  \begin{equation*}
    \bg_k = \phi'(\cdot) \paren*{ \Psi(\bx, \by_k) - \Psi(\bx, \by^*) }.
  \end{equation*}
  Since $\phi'$ is bounded by $1$ (Lipschitz property of $\wt \Phi$; see
  Remark~\ref{rem:variance-surrogates}) and
  the feature norm is bounded by $R$, the norm of any single estimate is
  bounded:
  \begin{equation*}
    \| \bg_k \|_2
    = | \phi'(\cdot) | \cdot \| \Psi(\bx, \by_k) - \Psi(\bx, \by^*) \|_2
    \le 1 \cdot (R + R) = 2R.
  \end{equation*}
  The total estimator is the average
  $\h \nabla \ell(\bw) = \frac{1}{K} \sum_{k=1}^K \bg_k$.  Using the standard
  variance property for independent random variables:
  \begin{align*}
    \E\bracket*{ \| \h \nabla \ell(\bw) - \nabla \ell(\bw) \|_2^2 }
    & = \frac{1}{K^2} \sum_{k=1}^K \E\bracket*{ \| \bg_k - \nabla \ell(\bw) \|_2^2 } \\
    & \le \frac{1}{K} \sup_{\by} \| \bg(\by) \|_2^2 \\
    & \le \frac{(2R)^2}{K} = \frac{4R^2}{K}.
  \end{align*}
  Thus, the variance decreases linearly with $K$ and is independent of the
  cardinality $|\sY|$.
\end{proof}

\begin{remark}[Other Linear-Core surrogates]
  \label{rem:variance-surrogates}
  The proof uses only $\sup_u |\phi'(u)| \leq 1$. For $\wt \Phi$ with $\Phi$
  non-decreasing (e.g., the logistic or exponential base), we have
  $\wt \Phi'(u) = -1$ for $u \leq 1$ and
  $\wt \Phi'(u) = -\Phi'(1 - u)/\Phi'(0) \in [-1, 0]$ for $u > 1$, since
  $\Phi'$ is non-negative and non-decreasing. For the symmetric $\ov \Phi$, the
  left branch has derivative $-\Phi'(-1 - u)/\Phi'(0)$ for $u < -1$, whose
  magnitude can exceed $1$. The same argument then gives the bound
  $4 G^2 R^2/K$, with $G = \sup_u |\ov \Phi'(u)|$, which is finite whenever
  $\Phi$ is Lipschitz (e.g., $G = 2$ for the logistic base, since
  $\Phi' \leq 1$ and $\Phi'(0) = 1/2$). In all cases, the bound is independent
  of $|\sY|$.
\end{remark}

\begin{remark}[Contrast with Sampled Softmax]
  While techniques like Sampled Softmax allow for $O(1)$ updates, they
  approximate the Log-Likelihood objective, which suffers from slower
  square-root consistency rates (see Table~\ref{tab:loss_comparison}). Our
  Stochastic Linear-Core approach is unique in that it combines $O(1)$ sampling
  efficiency with the fast linear consistency rates of margin-based losses.
\end{remark}

\subsection{Empirical Validation: Sequence Tagging Efficiency}
\label{sec:struct_experiments}

To empirically validate the computational efficiency claims discussed in
Section~\ref{sec:struct_optimization}, we compare our method against a standard
baseline on a sequence tagging task.

A critical practical limitation of standard structured prediction methods like
the Structured SVM (SSVM) is the computational cost of the training loop. The
SSVM objective, $\max_{y'\neq y} \max(0, \ell(y',y) - (h(x, y) - h(x, y')))$,
requires solving a \emph{loss-augmented inference} problem (finding the ``most
violated constraint'') at every gradient step. For sequence tagging tasks, the
standard evaluation metric is the Hamming loss, defined as
$\ell(y, y') = \frac{1}{L} \sum_{j=1}^{L} \1_{y'_j \neq y_j}$, where $L$ is the
sequence length and $y_j$ denotes the label at the $j$-th position. Optimizing
the SSVM with this loss necessitates running the Viterbi algorithm
\citep{viterbi2003error}, which scales as $O(L |\sY|^2)$ per sample and is
difficult to parallelize on modern hardware.

In contrast, our Structured Linear-Core Surrogate exploits the additive
structure described in Algorithm~\ref{alg:struct_sgd}
(Section~\ref{sec:struct_optimization}), where the loss is evaluated by
aggregating local margins rather than solving a global maximization
problem. This allows the objective to be optimized using simple stochastic
sampling, avoiding the sequential inference bottlenecks inherent to the SSVM.

\paragraph{Setup.}  We consider a synthetic sequence labeling task with sequence
length $L=20$, label set size $|\sY|=200$, and input dimension $d=20$. We
generate $1{,}000$ training sequences using a linear Hidden Markov Model
\citep{rabiner2002tutorial} with strong random transition potentials to ensure
that structural dependencies are significant. We train a linear neural sequence
model following the architecture of \citet{collobert2011natural}, consisting of
a linear projection for unary scores and a learnable transition matrix. We
optimize the model using Stochastic Gradient Descent (SGD)
\citep{robbins1951stochastic} with a fixed learning rate $\eta=0.01$ and a batch
size of $1$. We compare the wall-clock training time required to reach a target
test error for the Structured SVM \citep{tsochantaridis2005large} (implemented
with an exact Viterbi solver) against our Structured Linear-Core Surrogate
(implemented with the stochastic sampling strategy).

\paragraph{Convergence speed.}  Figure~\ref{fig:struct_efficiency} plots the test
Hamming error against wall-clock training time. The Structured SVM (Red) suffers
from the high overhead of the Viterbi oracle, resulting in slow convergence in
real time, requiring over 200 seconds to minimize the error. Our Linear-Core
Surrogate (Blue), using the smooth one-sided logistic tail $\wt \Phi_{\log}(u)$
and efficient additive decomposition, demonstrates a dramatic speedup. It
converges to the optimal error rate almost immediately (within the first few
seconds), validating that our method offers both the theoretical benefits of
consistency and the practical advantage of computational efficiency in
structured domains.

\paragraph{Scalability analysis.}  We further investigate the impact of
vocabulary size on training throughput. As noted previously, solving the
loss-augmented inference problem via Viterbi (for SSVM) imposes a quadratic
dependency $O(L|\sY|^2)$.  As shown in Figure~\ref{fig:scalability} (Red
curve), this quadratic complexity makes training prohibitively slow as the
vocabulary size grows; increasing $|\sY|$ from 100 to 400 results in a nearly
$8\times$ increase in training time for the Structured SVM.

In contrast, the Linear-Core surrogate proposed in this work is differentiable
everywhere. This smoothness property fundamentally changes the optimization
landscape: instead of solving a combinatorial maximization problem (argmax) at
every step, we can estimate the gradient as an expectation over the label
space. This allows us to use \emph{unbiased stochastic sampling}
(Algorithm~\ref{alg:struct_sgd}) to approximate the gradient, decoupling the
computational cost from the size of the output space. As shown in
Figure~\ref{fig:scalability} (Blue curve), our method maintains a constant
throughput regardless of vocabulary size. At $|\sY|=400$, the Linear-Core
surrogate achieves a \textbf{23$\times$ speedup} over the SSVM. This confirms
that our approach enables efficient linear-rate training on large-scale
structured problems where traditional max-oracle methods are intractable.

\begin{figure}[t]
  \centering
  \begin{minipage}[t]{0.48\linewidth}
    \centering
    \includegraphics[width=\linewidth]{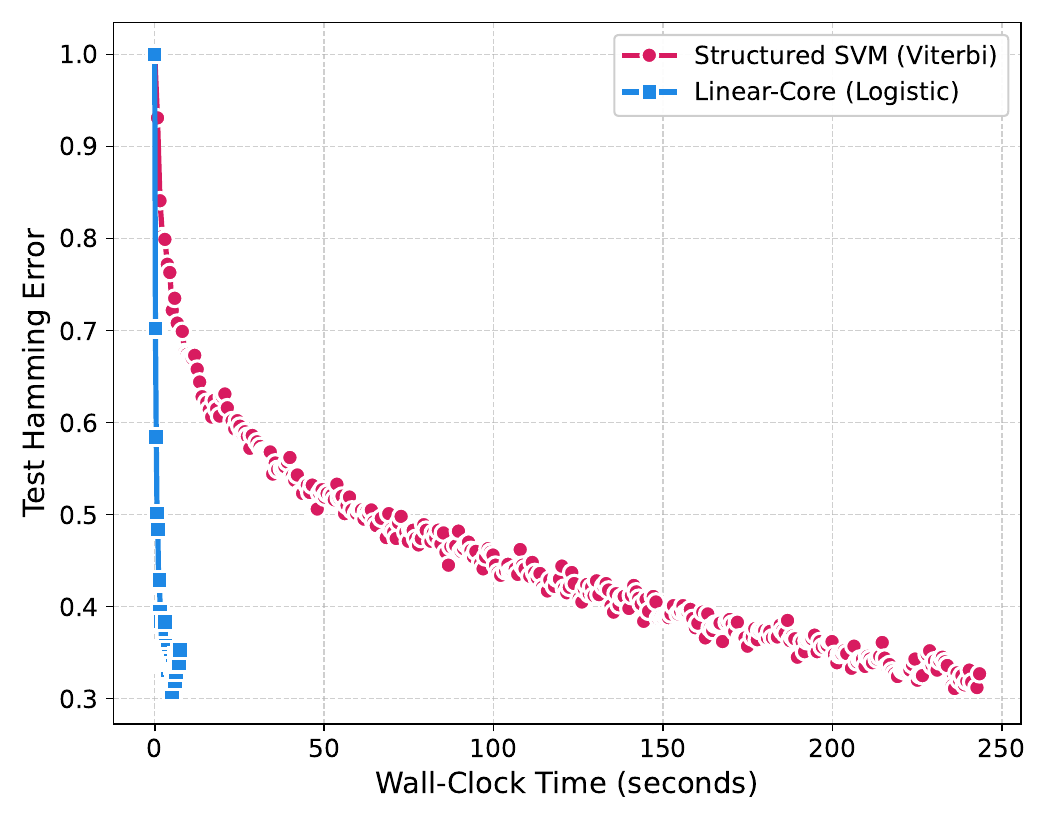}
    \caption{Test Error vs.\ Wall-Clock Time for Sequence Tagging.  We compare
      a Structured SVM (Red Circles) against our Logistic Linear-Core Surrogate
      (Blue Squares) on a synthetic sequence labeling task
      ($L=20, |\sY|=200$).  The Structured SVM is bottlenecked by the
      sequential Viterbi algorithm required for every gradient update.  In
      contrast, our method processes updates significantly faster due to the
      efficiency of the stochastic sampling approach, achieving comparable test
      error in a fraction of the wall-clock time (e.g., reaching $<0.35$ error
      in under 10 seconds versus $\approx 250$ seconds for SSVM).}
    \label{fig:struct_efficiency}
  \end{minipage}\hfill
  \begin{minipage}[t]{0.48\linewidth}
    \centering
    \includegraphics[width=\linewidth]{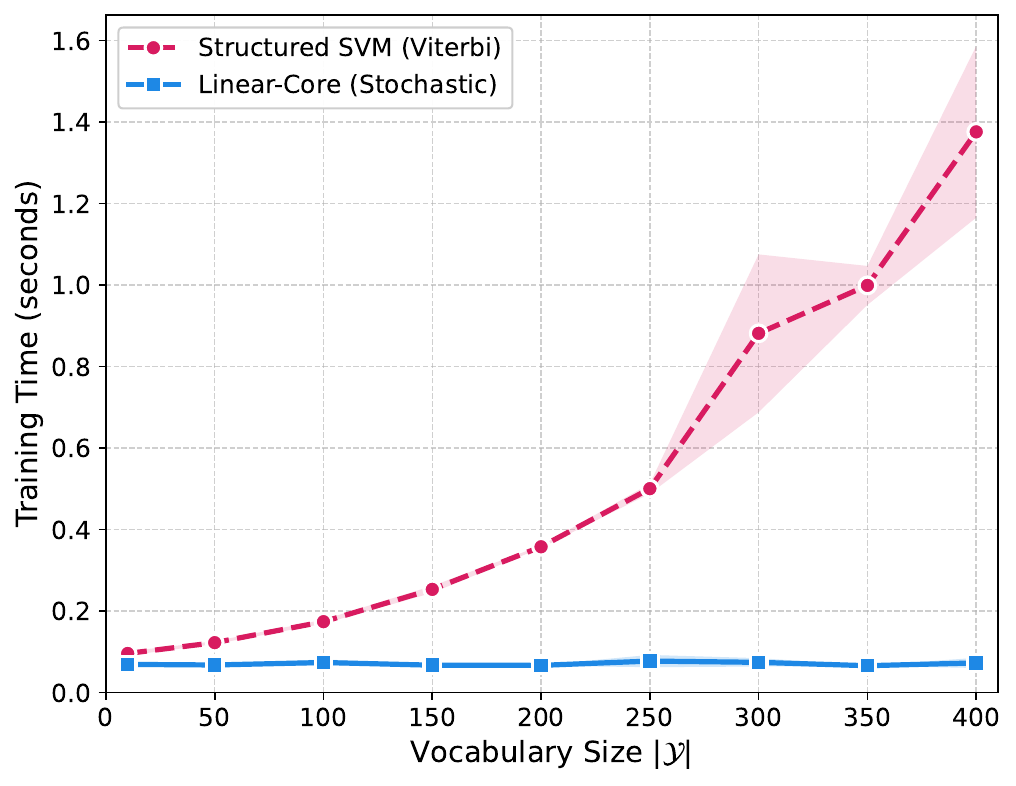}
    \caption{Scalability with Vocabulary Size.  We compare the wall-clock
      training time per batch of the Structured SVM (using Viterbi inference)
      against our Linear-Core Surrogate (using stochastic sampling) as the
      label vocabulary size $|\sY|$ increases.  The SSVM (Red) suffers from the
      quadratic complexity of dynamic programming ($O(L|\sY|^2)$).  In
      contrast, the Linear-Core surrogate enables unbiased gradient estimation
      via sampling, resulting in near-constant scaling ($O(L)$). At
      $|\sY|=400$, our method achieves a \textbf{23$\times$ speedup}, reducing
      the computational bottleneck of large-scale structured learning.}
    \label{fig:scalability}
  \end{minipage}
\end{figure}

It is important to note that \emph{Conditional Random Fields (CRF)}
\citep{lafferty2001conditional}, the standard probabilistic approach for
sequence modeling, shares the same computational bottleneck as SSVM. The
gradient of the CRF log-likelihood requires computing marginal probabilities via
the Forward-Backward algorithm, which also scales as $O(L|\sY|^2)$.

Furthermore, while widely used approaches such as CRF and SSVM are natural,
recent theoretical analysis has shown that their associated loss functions are
not Bayes-consistent with respect to discrete target losses, such as the Hamming
loss~\citep{MaoMohriZhong2023structured}. Consequently, these methods inherently
cannot be supported by the strong linear $\sH$-consistency bounds that we
establish for Linear-Core surrogates in Theorem~\ref{thm:bound-struct}.

\subsection{Real-World Efficiency: Fine-Grained POS Tagging}
\label{sec:efficiency_exp}

To rigorously quantify the computational advantage of our method, we conducted a
stress test designed to expose the quadratic bottleneck of the CRF. We used the
\emph{Penn Treebank (PTB)} Part-of-Speech tagging
dataset~\citep{marcus1993building} but simulated a fine-grained tagging task by
artificially inflating the tag set size to $|\sY| = 4000$. This simulates
complex morpho-syntactic tagging or open-domain sequence labeling tasks where
the label space is large.

\paragraph{Setup.}  To ensure that our efficiency claims hold perfectly
regardless of the underlying neural backbone, we compared methods across two
different architectures: a standard Bidirectional
LSTM~\citep{hochreiter1997long} and a modern BERT-based Transformer backbone
\citep{devlin2018bert}.

The baselines (BiLSTM-CRF and BERT-CRF) use the respective backbone followed by
a Conditional Random Field (CRF) layer~\citep{lafferty2001conditional}. These
models minimize the negative log-likelihood of the correct tag sequence,
computing gradients exactly via the Forward-Backward algorithm with a time
complexity of $O(L|\sY|^2)$ per sequence. We compare these against our proposed
BiLSTM-Linear-Core and BERT-Linear-Core, which use the same architectures but
minimize the Linear-Core surrogate loss using the stochastic sampling algorithm
(Algorithm~\ref{alg:struct_sgd}) described in
Section~\ref{sec:struct_optimization}. Crucially, this approach reduces the time
complexity to $O(L)$.

For the BiLSTM models, we used an embedding dimension of 128 and a hidden
dimension of 256. Training was performed using SGD with momentum $0.9$
\citep{nesterov1983method} (and AdamW \citep{loshchilovdecoupled} for
BERT). Critically, we restricted the batch size to $B=8$. This was necessitated
by the CRF baselines, which incur a prohibitive memory cost due to storing the
computation graph for $4000 \times 4000$ transition interactions at every
sequence step. Our method, having $O(L)$ memory complexity, could theoretically
support much larger batches, but we maintained the same batch size for a fair,
controlled comparison. To ensure the bottleneck was strictly computational, we
filtered the dataset to include only sequences with length $L \ge 100$.

\begin{figure}[t]
  \centering
  \includegraphics[width=0.44\linewidth]{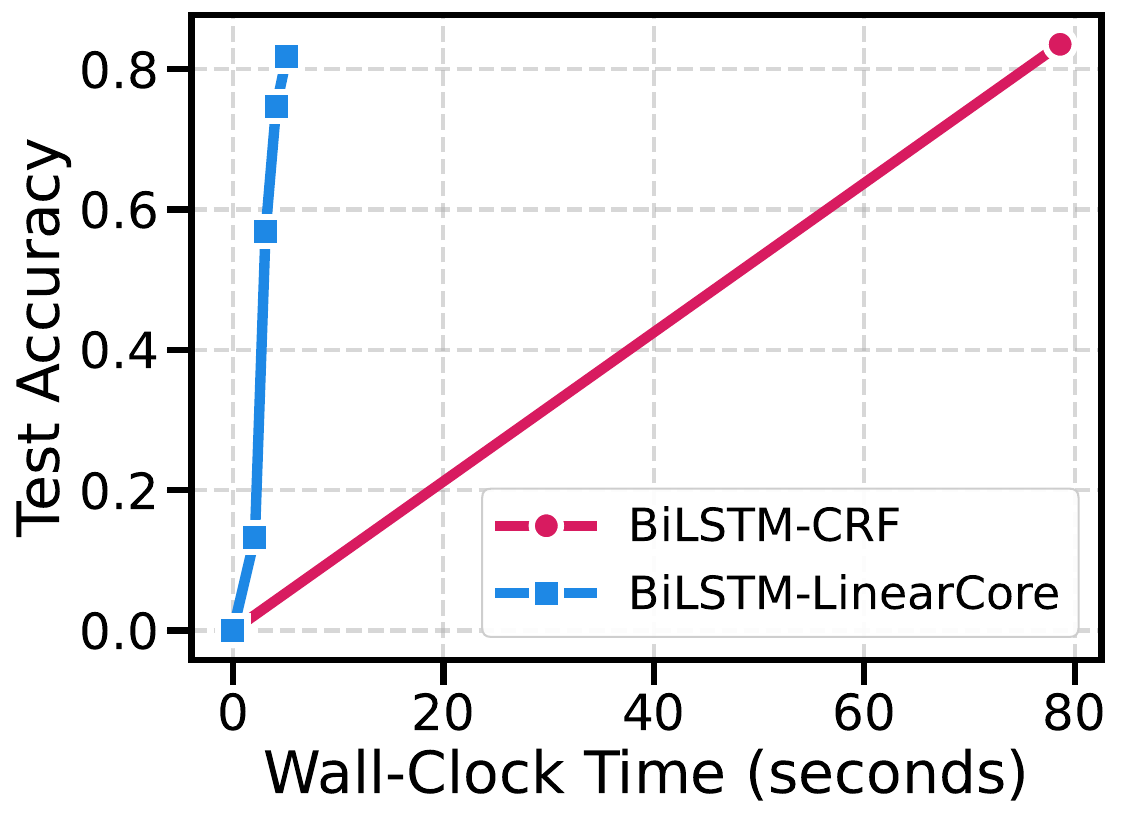}
  \hspace{0.04\linewidth}
  \includegraphics[width=0.44\linewidth]{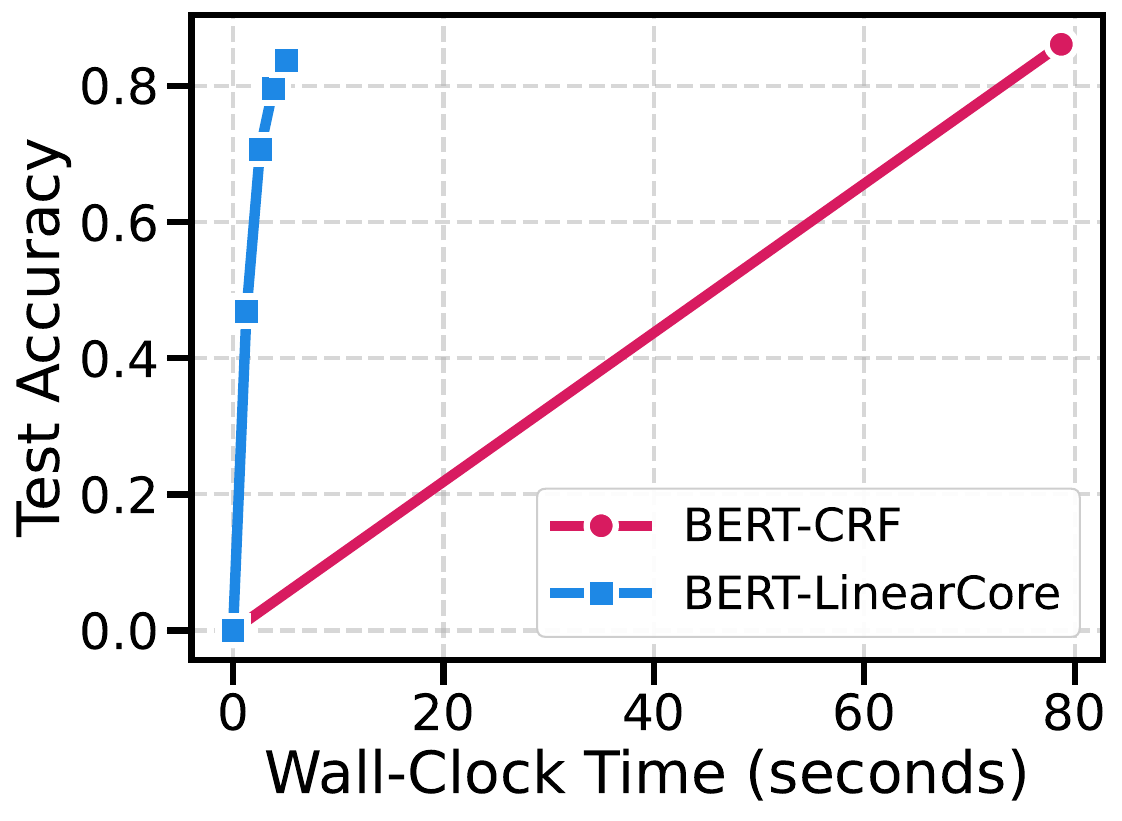}
  \caption{Real-world efficiency: Time to accuracy for sequence tagging
    ($|\sY|=4000$) using BiLSTM (Left) and BERT (Right) architectures.}
  \label{fig:efficiency}
\end{figure}

\paragraph{Results.}  We measured the wall-clock time to reach a target test
accuracy of 83\%. As shown in Figure~\ref{fig:efficiency}, the results are
dramatic.

For the BiLSTM architecture, the BiLSTM-CRF (Red) struggles with the
computational load. A single training epoch requires 78.4 seconds due to the
expensive matrix operations of the partition function. In contrast, the
BiLSTM-Linear-Core (Blue) converges rapidly. By using our stochastic sampling
algorithm, it bypasses the quadratic bottleneck entirely, reaching the target
accuracy in just 4.5 seconds. Quantitatively, our method achieves a
\textbf{17.4$\times$ speedup} in time-to-accuracy over the exact CRF baseline.

These findings perfectly mirror our results on the modern Transformer backbone,
proving that the computational bottleneck lies strictly in the structured loss
layer rather than the encoder. The BERT-CRF baseline required 78.7 seconds to
reach the target accuracy, remaining bottlenecked by the computation of the
partition function over the massive transition space. In contrast,
BERT-Linear-Core bypassed the quadratic bottleneck entirely, reaching the target
accuracy in just 5.1 seconds, which corresponds to a \textbf{15.4$\times$
  speedup} over the exact CRF inference.

These results confirm that Linear-Core Surrogates enable expressive structured
models in large-output domains where traditional CRFs are computationally
intractable. While standard English POS tagging has small tag sets, many
morphologically rich languages (Finnish, Turkish, Arabic) have tag sets scaling
into the thousands.  Our simulation with $|\sY| = 4000$ shows that Linear-Core
surrogates unlock efficient structured prediction for such complex tasks.

\subsection{Generalization to Non-Decomposable Metrics via Rational Kernels}
\label{sec:rational_kernels}

Our structured prediction analysis above centered on the Hamming loss due to its
immediate token-level additive decomposability. However, the theoretical and
computational advantages of the Linear-Core surrogate extend elegantly to
non-decomposable metrics such as exact F1-score and BLEU when viewed through the
lens of rational kernels \citep{cortes2004rational}.

While BLEU is often considered strictly non-decomposable at the token level, its
core mechanism, $n$-gram overlap counts, can be expressed exactly as an instance
of a rational kernel computed via weighted finite-state transducers (WFSTs). By
expanding the state space of our sequence model to encode length-$(n-1)$
histories, the $n$-gram matching scores $\ov{\ell}(y', y)$ become strictly
additively factorizable over the transitions of this expanded automaton.

Consequently, we do not need to rely on high-variance, sequence-level sampling
strategies (like REINFORCE). Instead, our unbiased stochastic gradient estimator
(Algorithm~\ref{alg:struct_sgd}) generalizes directly: we simply construct our
proposal distributions $\sD_1$ and $\sD_2$ over the paths of this expanded
WFST. Because the Linear-Core surrogate's consistency and variance bounds
(Theorem~\ref{thm:bound-struct} and Theorem~\ref{thm:variance_bound}) rely only
on the additive aggregation of margins, and not on the specific topology of the
state space, both the optimal $O(\Delta\sL)$ linear transfer rate and the $O(1)$
per-step computational efficiency are rigorously preserved.

%%%%%%%%%%%%%%%%%%%%%%%%%%%%%%%%%%%%%%%%%%%%%%%%%%%%%%%%%%%%%%%%%%%%%%%%%%%%%%%
\section{Related Work and Discussion}
\label{sec:related}
%%%%%%%%%%%%%%%%%%%%%%%%%%%%%%%%%%%%%%%%%%%%%%%%%%%%%%%%%%%%%%%%%%%%%%%%%%%%%%%

\paragraph{Comparison with Convolutional Fenchel--Young losses.}
\citet{cao2025establishing} also address the trade-off between smoothness and
linear convergence rates by introducing \emph{Convolutional Fenchel--Young
  losses}, defined via the infimal convolution of a generalized negentropy
regularizer and the target Bayes error. While their framework is theoretically
interesting, it differs from ours in several important respects.

\emph{(1) Implicit vs.\ explicit construction.}
The Fenchel--Young losses are defined \emph{implicitly} as the solution to a
variational optimization problem. An explicit closed-form expression is derived
only in the binary classification setting, where the resulting loss profile
effectively stitches a linear segment to a smooth tail---structurally matching
our construction. Indeed, in this special case their explicit binary loss is a
particular instance of our more general Linear-Core family (parameterized by an
arbitrary smooth base~$\Phi$ with $\Phi'(0) > 0$). In the multi-class and
structured settings, their framework does not yield explicit loss functions
amenable to direct evaluation; each gradient step requires solving an auxiliary
optimization problem. In contrast, our surrogates are fully explicit and admit
$O(1)$ evaluation and gradient computation without auxiliary solvers.

\emph{(2) Decoding and predictor form.}
The framework of \citet{cao2025establishing} relies on specific link functions
(derived from gradients of the convex conjugate) to decode scores into
predictions, departing from the standard $\argmax$ rule used in virtually all
deep learning applications. Our Linear-Core surrogates retain the standard
$\argmax$ decoding, $\hh(x) = \argmax_{y \in \sY} h(x, y)$, and serve as
direct, smooth replacements for the Hinge loss in any existing architecture.

\emph{(3) Structured prediction.}
To the best of our knowledge, the framework of \citet{cao2025establishing} does
not address the computational challenges of structured prediction where the
output space $\sY$ is exponentially large. In such settings, generic discrete
bounds are insufficient as they may hide dependencies on $|\sY|$ that make
computation intractable. Our work explicitly extends linear $\sH$-consistency
analysis to the structured setting (Theorem~\ref{thm:bound-struct}), proves
that convexity and smoothness are preserved, and introduces a stochastic
gradient algorithm (Algorithm~\ref{alg:struct_sgd}) that reduces the per-step
complexity from $O(|\sY|^2)$ to $O(1)$ per position.

\emph{(4) Stronger $\sH$-consistency guarantees.}
The regret bounds established by \citet{cao2025establishing} are limited to
excess loss bounds with respect to the family of all measurable functions (Bayes
consistency). Such bounds guarantee convergence to the Bayes optimal classifier
only in the non-parametric limit. In contrast, we provide \emph{linear
  $\sH$-consistency bounds}
\citep{awasthi2022multi,MaoMohriMohriZhong2023twostage,MaoMohriZhong2023characterization,MaoMohriZhong2023ranking,MaoMohriZhong2023rankingabs,mao2023cross,MaoMohriZhong2024deferral,MaoMohriZhong2024predictor,MaoMohriZhong2024score,mao2024h,mao2024multi,mao2024realizable,mao2024regression,MohriAndorChoiCollinsMaoZhong2024learning,cortes2024cardinality,cortes2025balancing,cortes2025improved,MaoMohriZhong2025mastering,MaoMohriZhong2025principled,mao2025enhanced,mao2025theory,zhong2025fundamental,desalvo2025budgeted,CortesMaoMohriZhong2026defid,CortesMohriZhong2026mod,montreuil2025adversarial,montreuil2025two,montreuil2026beyond,montreuil2026why,montreuil2026online,montreuil2026optimal,montreuil2026adversarial,montreuil2026learning,montreuil2026time,MohriZhong2026rllm,mohri2025beyond,mohri2026principled,mohri2026generalized}. These
bounds explicitly relate the estimation error of the surrogate within a
restricted hypothesis set $\sH$ to the estimation error of the target loss.
This is a strictly stronger guarantee that remains valid even for misspecified
models (where $\sH$ does not contain the Bayes optimal classifier), ensuring
that optimizing the smooth Linear-Core Surrogate effectively minimizes the
regret against the best possible competitor in~$\sH$.

\paragraph{Relationship with the universal square-root growth rate.}
Our results may appear at first glance to contradict the universal square-root
growth rate established by \citet*{MaoMohriZhong2024}, who proved that
$\sH$-consistency bounds for ``smooth'' margin-based surrogates necessarily
exhibit a square-root growth rate $\sT(t) = \Theta(t^2)$. However, there is no
contradiction: the definition of smoothness in their work requires $\Phi$ to be
\emph{twice continuously differentiable with $\Phi''(0) > 0$}, i.e., with
\emph{non-vanishing curvature at the origin}. This condition is satisfied by all
standard smooth losses (logistic, exponential, squared hinge), but it is
\emph{violated by design} in our Linear-Core surrogates: our construction
guarantees $\ov\Phi''(0) = 0$ (or, more precisely, $\ov\Phi$ is affine on
$[-1,1]$, so the second derivative vanishes on the entire linear core). This
vanishing curvature is precisely the mechanism that allows the transformation
$\sT(t)$ to grow linearly rather than quadratically near zero
(Lemma~\ref{lemma:lin-binary}), thereby escaping the square-root barrier. In
other words, the result of \citet{MaoMohriZhong2024} characterizes a fundamental
limitation of losses with non-zero curvature at the decision boundary, and our
Linear-Core surrogates circumvent this limitation by design.

\paragraph{Contrast with Huber smoothing.}
It is important to distinguish our approach from standard Huber smoothing
\citep{huber1964robust}.  In the context of classification, Huber-style
smoothing typically replaces the non-differentiable ``kink'' of the Hinge loss
(or the region near the decision boundary) with a quadratic segment to ensure
differentiability.
While effective for optimization, replacing the linear segment with a quadratic
destroys the local linearity at the origin.  This structural change generally
degrades the $\sH$-consistency bound from a fast linear rate to a slower
square-root rate (similar to the Squared-Hinge loss)
\citep{Awasthi2022Hconsistency}.  In contrast, our method explicitly
\emph{preserves} the linearity at the origin (specifically in the interval
$[-1, 1]$).  This retention of the ``linear core'' is precisely what enables the
fast rates derived in Theorem~\ref{thm:bound-binary} and
Corollary~\ref{cor:bound-binary-one}, while the smoothing is applied only to the
tails to facilitate gradient-based optimization.

%%%%%%%%%%%%%%%%%%%%%%%%%%%%%%%%%%%%%%%%%%%%%%%%%%%%%%%%%%%%%%%%%%%%%%%%%%%%%%%
\section{Conclusion}
\label{sec:conclusion}
%%%%%%%%%%%%%%%%%%%%%%%%%%%%%%%%%%%%%%%%%%%%%%%%%%%%%%%%%%%%%%%%%%%%%%%%%%%%%%%

We introduced Linear-Core Surrogates, the first explicit family of smooth
convex loss functions that provably achieve linear $\sH$-consistency bounds.
By resolving the long-standing dichotomy between smoothness and fast
consistency rates, these surrogates simultaneously enable the fast $O(1/T)$
optimization convergence of differentiable losses and the optimal linear
transfer rate previously available only to non-smooth losses like the Hinge
loss. We established these guarantees progressively across binary,
multi-class, and structured prediction, demonstrating the generality of the
construction.  Beyond their theoretical significance, Linear-Core Surrogates
yield practical benefits: natural robustness to instance-dependent label noise
in multi-class classification, and a stochastic gradient algorithm for
structured prediction that bypasses the quadratic complexity of exact
inference, delivering massive speedups over standard baselines.

\bibliographystyle{abbrvnat}
\bibliography{slin}

\end{document}